\documentclass[oneeqnum, onethmnum, onefignum, onetabnum]{article}

\usepackage{amssymb, bigints,xcolor, multirow, array, cancel, amsthm,url}
\usepackage{graphicx}
\usepackage[top=1.5in, bottom=1.25in, left=1.25in, right=1.25in]{geometry}

\newtheorem{example}{Example}
\newtheorem{definition}{Definition}
\newtheorem{theorem}{Theorem}
\newtheorem{lemma}{Lemma}
\newtheorem{proposition}{Proposition}

\newcommand{\R}{\mathbb{R}}
\newcommand{\Rinf}{\overline{\R}}
\newcommand{\sign}{\text{sign\,}}
\newcommand{\prox}{\text{prox}}
\newcommand{\proj}{\text{proj}}
\newcommand{\dom}{\text{dom}\, }
\newcommand{\inter}{\text{int}\, }

\newcommand{\diver}{\text{div}}

\newcommand{\TV}{\text{TV}}
\newcommand{\VTV}{\text{VTV}}
\newcommand{\CTV}{\text{CTV}}
\newcommand{\diag}{\mathop{\mathrm{diag}}}

\title{Collaborative Total Variation: \\A General Framework for Vectorial TV Models\thanks{This work was supported by the Ministerio de Ciencia e Innovaci\'on under grant TIN2011-27539, and by ERC Starting Grant \textquotedblleft Convex Vision\textquotedblright.}}

\author{J. Duran\footnotemark[2] \and M. Moeller\footnotemark[3] \and C. Sbert\footnotemark[2] \and D. Cremers\footnotemark[3]}

\begin{document}

\maketitle

\renewcommand{\thefootnote}{\fnsymbol{footnote}}

\footnotetext[2]{Universitat de les Illes Balears, Department of Mathematics and Computer Science, Anselm Turmeda, Ctra. de Valldemossa km. 7.5, 07122 Palma de Mallorca, Spain (joan.duran@uib.es, catalina.sbert@uib.es). During this work, J. Duran had a fellowship of the Conselleria d'Educaci\'o, Cultura i Universitats of the Govern de les Illes Balears for the realization of his Ph.D. thesis, which was selected under an operational program co-financed by the European Social Fund.}
\footnotetext[3]{Technische Universit{\"a}t M{\"u}nchen, Department of Mathematics and Computer Science, Informatik 9, Boltzmannstrasse 3, 85748 Garching, Germany (michael.moeller@in.tum.de, cremers@tum.de).}
\renewcommand{\thefootnote}{\arabic{footnote}}

\begin{abstract}
Even after over two decades, the total variation (TV) remains one of the most popular regularizations for image processing problems and has sparked a tremendous amount of research, particularly to move from scalar to vector-valued functions. In this paper, we consider the gradient of a color image as a three dimensional matrix or tensor with dimensions corresponding to the spatial extend, the differences to other pixels, and the spectral channels. The smoothness of this tensor is then measured by taking different norms along the different dimensions. Depending on the type of these norms one obtains very different properties of the regularization, leading to novel models for color images. We call this class of regularizations {\it collaborative total variation} (CTV). On the theoretical side, we characterize the dual norm, the subdifferential and the proximal mapping of the proposed regularizers. We further prove, with the help of the generalized concept of singular vectors, that an $\ell^{\infty}$ channel coupling makes the most prior assumptions and has the greatest potential to reduce color artifacts. Our practical contributions consist of an extensive experimental section where we compare the performance of a large number of collaborative TV methods for inverse problems like denoising, deblurring and inpainting.
\end{abstract}


\pagestyle{myheadings}
\thispagestyle{plain}

\section{Introduction}

Many problems in image processing require the choice of a good prior that makes assumptions on the structure of the underlying image we seek to estimate. This prior often takes the form of a regularization term for an energy functional which is to be minimized. Observing that quadratic regularization did not allow recovering sharp discontinuities, Rudin, Osher and Fatemi proposed the total variation (TV) penalty \cite{ROF1992} for solving inverse problems. The total variation pioneered as a discontinuity-preserving regularizer in the sense that it assigns the same energy cost to sharp and smooth transitions. Therefore, it is one of the simplest (convex) variational models that allows discontinuities, yet it disfavours the solution to have oscillations.

Although the TV was originally designed for image denoising, it has become one of the most popular regularizations for many image processing problems and has sparked a tremendous amount of research.  While many extensions like anisotropic TV \cite{Esedoglu2004, Grasmair2010, Weickert1998}, weighted TV \cite{CollDuran2015, GilboaSochen2006, Grasmair2009}, higher order TV \cite{Benning2013, Bredies2010, ChanMarquina2000, Papafitsoros2014, You1996}, nonlocal TV \cite{DuranMoellerNLTV2015, GilboaOsher2007, GilboaOsher2008, Peyre2008, Ranftl2014}, or nonconvex TV \cite{Krishnan2009, Mollenhoff2015} have been proposed, the general idea of penalizing image oscillations with one-homogeneous functions depending on the spatial derivatives of the image remain the same. A lot of recent research has focused on extending the classical TV model for grayscale images to vector-valued (color or multichannel) images. We provide below an initial overview on vectorial total variation, which will be detailed and link to our framework in Section \ref{sec:VTV}.

\subsection{Vector Valued Total Variation}

Let $\Omega\subset\R^M$ be a bounded domain, then the scalar total variation  of a locally integrable function $u\in\mathcal{L}^1_{\text{loc}}\left(\Omega,\R\right)$ is
\begin{equation}\label{scalarTV}
\TV(u):=\sup_{\xi\in\Xi}\left\lbrace \int_{\Omega} u(x) \: \diver\left(\xi(x)\right) \, dx \right\rbrace,
\end{equation}
where $x=(x_1, \ldots, x_M)\in\Omega$ and 
\begin{equation}\label{scalarTVXi}
\Xi = \left\lbrace \xi\in\mathcal{C}_c^1(\Omega,\R^M) \: : \: \|\xi(x)\|\leq 1, \forall x\in\Omega\right\rbrace
\end{equation}
is the set of continuously differentiable and bounded functions with compact support in $\Omega$. The definition given in \eqref{scalarTV} introduces a dual formulation according to which the TV is the convex conjugate of the indicator function of the convex set
$K_{\TV} := \left\lbrace \diver(\xi) \: : \: \xi\in \Xi\right\rbrace$. For a differentiable function $u\in \mathcal{C}^1(\Omega, \mathbb{R})$, one has $\text{TV}(u)=\int_{\Omega} |\nabla u(x)|\,dx$. Note that the TV can be defined differently depending on the norm used in \eqref{scalarTVXi}. For a better understanding, let us restrict ourselves to $u\in \mathcal{C}^1(\Omega, \mathbb{R})$ and denote its gradient by $\nabla u(x)=\left(\partial_{x_1} u(x), \ldots, \partial_{x_M} u(x)\right)\in\R^M$ at each $x\in\Omega$. Therefore, using $\|\cdot\|_2$ as dual norm leads to the isotropic TV, $\int_{\Omega}  \sqrt{ \sum_{m} \left(\partial_{x_m} u(x) \right)^2}\, dx$, whereas the anisotropic TV follows from choosing $\|\cdot\|_{\infty}$ in \eqref{scalarTVXi}, $\int_{\Omega} \sum_{m} |\partial_{x_m} u(x)| \, dx$.

The idea of the vectorial total variation is to extend the above definitions to vector-valued functions $\mathbf{u}:\Omega\rightarrow \R^C$. A major decision with color images is how to couple channels. A straightforward approach proposed by Blomgren and Chan \cite{Blomgren1998} consists in using a global channel coupling by penalizing the $\ell^2$ norm of the TV contributions across channels. However, local coupling outperforms global coupling in many theoretical and practical aspects \cite{Holt2014}. In this setting, most of the methods in color image reconstruction used an $\ell^1$ or $\ell^2$ norm to penalize the TV of the channels at each pixel \cite{Attouch2006, BressonChan2008, DuvalAujol2009}. Additionally, some interesting approaches incorporated a change of color space \cite{ChanKang2001, Condat2012}. Further versions of vectorial TV in literature are based on the singular values of the submatrices one obtains by fixing a pixel location and looking at the remaining matrix in the channel and derivative dimensions. Important cases are the Schatten$-\infty$ norm \cite{Goldluecke2012}, which penalizes the largest singular value, and the nuclear norm or Schatten$-1$ norm \cite{Lefkimmiatis2013},  which is a convex relaxation of minimizing the rank of the image Jacobian at each pixel \cite{Recht2010}.

\subsection{Problem Formulation}

For the sake of simplicity, we will consider discrete versions of the TV for the remainder of this paper. Let us define the Euclidean spaces $X:=\R^{N\times C}$ and $Y:=\R^{N\times M \times C}$, where $N$ is the number of pixels of the image, $M$ is the number of directional derivatives, and $C$ is the number of color channels. We thus consider a color image as a two-dimensional matrix of size $N\times C$ denoted by $\mathbf{u} = \left(\mathbf{u}_1, \ldots, \mathbf{u}_C\right)\in X$, with $\mathbf{u}_k = \left(u_{1,k}, \ldots, u_{N,k}\right)^{\top}\in\R^N$ for each channel $k\in\{1,\ldots, C\}$. On the other hand, we define the linear operator $K:X\rightarrow Y$ such that $K\mathbf{u}\in Y$ is a three-dimensional matrix or tensor. In the rest of the paper, we use the colon to denote all elements along one dimension. For example, the $\ell^p$ norm of $A\in Y$ with respect to its third dimension reads $\|A_{i,j,:}\|_p^p := \sum_{k} |A_{i,j,k}|^p$.

The general problem we are concerned with is
\begin{equation}\label{eq:minproblem}
\min_{\mathbf{u}\in X} G(\mathbf{u}) + \|K\mathbf{u}\|_{\vec{b},a}
\end{equation}
where $G:X\rightarrow\R$ is a proper, convex, l.s.c. functional and $\|\cdot\|_{\vec{b},a}$ is a collaborative sparsity enforcing norm penalizing the gradient of the color image to be detailed later. 

In this paper, we propose a general and intuitive framework that allows us not only to handle pre-existing vectorial total variation models, but also to introduce some new interesting regularizations for color image processing. Our idea is that, in a discrete setting, the gradient of a vector-valued image is nothing but a three dimensional matrix or tensor with the dimensions corresponding to the spatial extend, the directional derivatives considered as linear operators containing the differences to other pixels, and the color channels. The energy or smoothness of this tensor can be measured by taking different norms along the different dimensions. Depending on the types of norms one obtains very different properties of the regularization. 

Two relevant examples immediately arise from the proposed framework. For the sake of clarity, let us write $A:=K\mathbf{u}\in Y$. If we first take the $\ell^p$ norm to the color dimension, then the $\ell^q$ norm along the derivative dimension of the remaining 2D matrix and, finally, the $\ell^r$ norm to the final pixel vector, one obtains the $\ell^{p,q,r}$ norm:
\begin{equation}\label{eq:lpqr}
 \|A\|_{p,q,r} := \left( \sum_{i=1}^N \left( \sum_{j=1}^M \left( \sum_{k=1}^C |A_{i,j,k}|^p\right)^{q/p}\right)^{r/q} \right)^{1/r}.
\end{equation}
In \eqref{eq:lpqr}, any of the indices $p$, $q$ or $r$ being equal to infinity means taking the maximum of the absolute values along the corresponding dimension. A second important example consists of penalizing with the $\ell^p$ norm the singular values of the 2D matrices arising from each pixel (that is, the Schatten$-p$ norm), and then applying the $\ell^q$ norm along the remaining vector:
\begin{equation}\label{eq:Splq}
(S^p, \ell^q)(A):= \left( \sum_{i=1}^N \left\| \left( \begin{array}{ccc} A_{i,1,1} & \cdots & A_{i,1,C} \\ \vdots  & \ddots & \vdots \\ A_{i,M,1} & \cdots & A_{i,M,C}\end{array}  \right) \right\|^q_{S^p} \right)^{1/q}.
\end{equation}

As an illustrative example, Figure \ref{fig:SyntheticEdge} shows the results of a numerical experiment regarding the ability of different channel couplings to suppress color artifacts. We use a synthetic image where we leave open if the colored wave pattern is signal content or noise. We see that the channel-by-channel regularization due to the $\ell^1$ norm eliminates all noise from constant regions but the color structure of the underlying image is not touched.  On the contrary, the $\ell^{\infty}$ norm leads to the strongest channel coupling and is able to remove the color oscillations completely. In between both, the $\ell^2$ channel coupling significantly reduces the colors around the white square but does not eliminate them. Therefore, we expect a color coupling with an $\ell^p$ norm to be stronger the larger $p$ is.

\begin{figure}[!htpb]
\centering
\begin{tabular}{cccc}
   \includegraphics[trim=1cm 1cm 1cm 1cm, clip=true, scale=0.403]{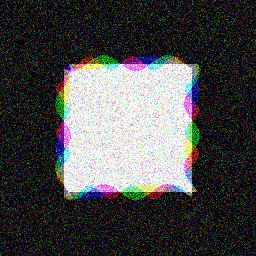} &
   \includegraphics[trim=1cm 1cm 1cm 1cm, clip=true, scale=0.403]{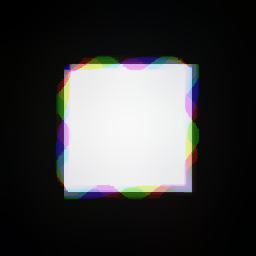} &
   \includegraphics[trim=1cm 1cm 1cm 1cm, clip=true, scale=0.403]{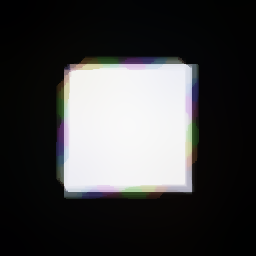} &
   \includegraphics[trim=1cm 1cm 1cm 1cm, clip=true, scale=0.403]{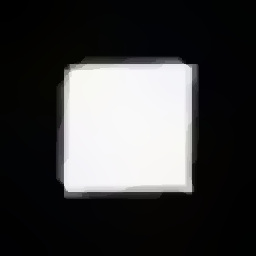} \\
  Noisy & $\ell^1$ coupling & $\ell^2$ coupling & $\ell^{\infty}$ coupling
\end{tabular}
\caption{Denoising a synthetic image where we leave open if the colored wave pattern is signal content or noise. One observes that uncoupling channels ($\ell^1$ norm) keeps the colored waves but the strongest channel coupling ($\ell^{\infty}$ norm) eliminates them. In between both approaches, the $\ell^2$ channel coupling suppresses but does not eliminate the wave pattern. }
\label{fig:SyntheticEdge}
\end{figure}

\subsection{Contributions and Preliminary Works}

We streamline below the novelty of our approach. The major contributions of this work are:
\begin{itemize}
\item The introduction of a large family of (discrete) convex energy functionals that generalize the TV to vector-valued images. Motivated by recent advances in compressed sensing, we interpret the total variation as looking for an image for which the gradient is sparse. We use collaborative sparsity \cite{YuanLin2006} to model different types of TV which are then used in a variational formulation to provide regularized solutions of ill-posed inverse problems in color imaging. We call this family of regularizers {\it Collaborative Total Variation} (CTV).

\item The definition of general collaborative sparsity enforcing norms that characterize all CTV regularizations. We further compute their dual norms and their subdifferentials, which play a direct role in computing optimality conditions of several regularized problems.

\item The proof, with the help of the generalized concept of singular vectors \cite{BenningBurger2013}, that an $\ell^{\infty}$ channel coupling leads to the strongest correlation, makes the most prior assumptions and has the greatest potential to reduce color artifacts.

\item The proposal of sophisticated collaborative norms such as $\ell^{\infty,1,1}$, $\ell^{\infty,2,1}$, $\ell^{2,\infty,1}$, and  $\ell^{\infty,\infty,1}$, which lead to novel methods for color images. All variants can be solved very efficiently by using the same splitting scheme, for instance, the primal-dual hybrid gradient (PDHG) method \cite{ChambollePock2011, EsserZhang2010, Zhu2008}. Since the key to obtaining a fast PDHG algorithm is an efficient evaluation of the proximity operators, they are provided in detail.

\item An extensive experimental evaluation of some of the proposed CTV methods on several image processing problems, such as denoising, deblurring or inpainting. A detailed performance comparison on different databases for color image denoising using the ROF model together with the proposed collaborative TV regularizations is provided in the companion paper \cite{DuranMoellerIPOL2015}. We further include some experiments for cartoon and texture decomposition. Code and an online demo to reproduce all examples will be made available soon. 
\end{itemize}

In the original conference paper \cite{DuranMoellerNLTV2015}, which contains preliminary parts of this work, we proposed to penalize the $\ell^{p,q,r}$ norm of the three-dimensional structure underlying the nonlocal gradient for color image reconstruction. In particular, the newly proposed $\ell^{\infty,1,1}-$NLTV model yielded superior results. In the current paper, we extend the original framework in order to include more general collaborative norms: we propose novel $\ell^{p,q,r}$ norms and further incorporate Schatten $(S^p, \ell^q)$ norms. We also provide a mathematical justification of the superiority of the $\ell^{\infty}$ coupling for restoring high inter-channel correlated images, as well we develop general properties for collaborative norms useful in optimization. Finally, we give a detailed performance comparison of more vectorial TV methods derived from the proposed framework in additional image processing problems.

During the wording of this work, the conference paper by Miyata and Sakai \cite{Miyata2012}, which pioneered the $\ell^{\infty}$ channel coupling, came to our hands. To the best of our knowledge,  \cite{Miyata2012} is the only paper that uses the supremum norm for vectorial TV. However, the authors proposed to first perform a color transform that reduces the inter-channel correlation. From our point of view, this change of color space is counter-intuitive when combined with the strong inter-channel coupling of $\ell^{\infty}$. One of our main contributions is to introduce the $\ell^{\infty}$ norm in a straightforward way and efficiently exploit its properties.

\subsection{Outline of the Paper}

The rest of the paper is organized as follows. The next section introduces the definition of collaborative norms and develops some general properties which play a direct role when computing optimality conditions of regularized problems. In Section \ref{sec:VTV} we summarize the current literature on different definitions for extending the TV to multichannel images. All of them are analyzed as special cases of the proposed approach. We investigate in Section \ref{sec:singVectors} which channel coupling leads to the strongest correlation, makes the most prior assumptions and has the greatest potential to reduce color artifacts. In Section \ref{sec:numerics}, we give detailed explanations on how to determine minimizers of typical image processing problems using CTV as a prior. Particularly, we write down the proximity operators for all types of regularizations discussed in this paper. We compare different CTV methods in numerical experiments for denoising, deblurring and inpainting of color images in Section \ref{sec:results}, before we draw conclusions in Section \ref{sec:conclusions}.

\section{Collaborative Total Variation Regularization}\label{sec:CTV}

In the following, we introduce a novel regularization family which we use to solve inverse problems in vector-valued image processing within a variational setting. The proposed models are based on the use of collaborative sparsity enforcing norms, which will be abbreviated as {\it collaborative norms}, that are defined below.

\subsection{Definition of Collaborative Norms}

By considering the derivatives of a color image as a linear operator, one obtains a three-dimensional matrix or tensor with one dimension corresponding to the pixels in the image, one dimension corresponding to the directional derivatives, and one dimension corresponding to the color channels. 

\begin{example}
For illustrative purposes, suppose that a color image given on a rectangular domain of size $N_{w}\times N_{h}$ has been rearranged from left to right and from top to bottom into a matrix $\mathbf{u} = (\mathbf{u}_1, \mathbf{u}_2, \mathbf{u}_3)\in \R^{N\times 3}$. Consider $K$ to be the standard gradient computed via forward differences along $x-$ and $y-$directions. Then, the two-dimensional submatrix obtained by fixing the $n-$th pixel in the first dimension is
$$
\left(
\begin{array}{ccc}
u_{n+1,1} - u_{n,1} & u_{n+1,2} - u_{n,2} & u_{n+1,3} - u_{n,3} \\
u_{n+N_w,1} - u_{n,1} & u_{n+N_w,2} - u_{n,2} &  u_{n+N_w,3} - u_{n,3}\\
\end{array}
\right).
$$
\end{example}

\begin{example}
Let us see how a neighbourhood filter fits in our framework for a color image with four pixels. Let $K$ be defined as the nonlocal gradient with respect to a weighting function $w$, which measures the similarity between two pixels in the image. In this case, we have $N=4$, $M=4$, and $C=3$. Contrary to the previous example, we fix here the color dimension to the $k-$th channel. Therefore, the remaining two-dimensional submatrix along pixel and derivative dimensions is
$$
\left(
\begin{array}{cccc}
0 & w_{1,2} \left(u_{1,k}-u_{2,k}\right) & w_{1,3} \left(u_{1,k}-u_{3,k}\right) & w_{1,4} \left(u_{1,k}-u_{4,k}\right)\\
w_{2,1} \left(u_{2,k}-u_{1,k}\right) & 0 & w_{2,3} \left(u_{2,k}-u_{3,k}\right) & w_{2,4} \left(u_{2,k}-u_{4,k}\right)\\
w_{3,1} \left(u_{3,k}-u_{1,k}\right) & w_{3,2} \left(u_{3,k}-u_{2,k}\right) & 0 & w_{3,4} \left(u_{3,k}-u_{4,k}\right)\\
w_{4,1} \left(u_{4,k}-u_{1,k}\right) & w_{4,2} \left(u_{4,k}-u_{2,k}\right) & w_{4,3} \left(u_{4,k}-u_{3,k}\right) & 0
\end{array}
\right).
$$
In general, the previous matrix is of size $N\times N$. However, one usually uses a few nonzero weights in practical applications.
\end{example}

Although in the literature only $\ell^1$ and $\ell^2$ norms have been mainly used so far, it makes sense to look at vectorial TV as applying the more and more popular mixed norms to the gradient of the image (see \cite{EsserMoeller2013, HeinsMoeller2014, Kowalski2009} and references therein). For a general tensor $A\in Y$, we introduce the following family of norms that we call {\it collaborative norms}.

\begin{definition}
Let $\|\cdot\|_a:\R^{N}\rightarrow\R$ be any vector norm and $\|\cdot\|_{\vec{b}}:\R^{M\times C}\rightarrow\R$ any matrix norm. Then, the \emph{collaborative norm} of $A\in\R^{N\times M\times C}$,  which will be denoted by $\|\cdot\|_{\vec{b},a}:\R^{N\times M\times C}\rightarrow\R$, is defined as
\begin{equation}\label{eq:matrixnorm}
\|A\|_{\vec{b},a}:= \|v\|_a, \quad \text{with}\quad v_i=\|A_{i,:,;}\|_{\vec{b}}, \:\: \forall i\in\{1,\ldots, N\},
\end{equation}
where $A_{i,:,:}$ is the (two-dimensional) submatrix obtained by staking the second and third dimensions of $A$ at each $i-$th position in the first dimension.
\end{definition}

We note that the examples given in \eqref{eq:lpqr} and \eqref{eq:Splq} follow from the above definition. Indeed, the $\ell^{p,q,r}$ norm arises from taking $\|\cdot\|_{\vec{b}}$ as the matrix $\ell^{p,q}$ norm and $\|\cdot\|_a$ as the $\ell^r$ norm. On the contrary, the $(S^p,\ell^q)$ norm is obtained when one considers $\|\cdot\|_{\vec{b}}$ to be the matrix Schatten-$p$ norm, that is penalizing the $\ell^p$ norm of the singular values of the submatrix $A_{i,:,:}$, and $\|\cdot\|_a$, the $\ell^q$ norm.

Since the collaborative norms defined in \eqref{eq:matrixnorm} are non invariant to permutations of the dimensions, we propose to denote $\|A\|_{\vec{b},a}(col, der, pix)$ for first applying the matrix norm $\|\cdot\|_{\vec{b}}$ to the submatrix obtained by fixing each pixel and looking at the remaining derivative and channel dimensions, and then using the vectorial norm $\|\cdot\|_a$ along the pixel dimension. Importantly, note that our framework covers any transform along each of the dimensions, in particular, allows us to incorporate color space transforms before applying any collaborative norm.

\subsection{General Properties of Collaborative Norms}

It is well known that duality plays a direct role in computing optimality conditions of several regularized problems. The following result characterizes the dual norm to any collaborative norm.

\begin{theorem}
\label{thm:dualnorm}
Let $\|\cdot\|_{\vec{b}^{*}}$ and $\|\cdot\|_{a^{*}}$ denote the dual norms to $\|\cdot\|_{\vec{b}}$ and $\|\cdot\|_{a}$, respectively. Consider $A\in\R^{N\times M\times C}$ and define $v\in\R^N$ such that $v_i:=\|A_{i,:,:}\|_{\vec{b}^{*}}$ for each $i\in\{1,\ldots, N\}$. If $\|v\|_{a^{*}}$ only depends on the absolute values of $v_i'$s, then the dual norm to $\|\cdot\|_{\vec{b},a}$, denoted by $\|\cdot\|_{\vec{b}^{*},a^{*}}$, is
\begin{equation}\label{eq:matrixnormdual}
\|A\|_{\vec{b}^{*},a^{*}}= \|v\|_{a^{*}}, \quad \text{with}\quad v_i=\|A_{i,:,;}\|_{\vec{b}^{*}}, \:\: \forall i\in\{1,\ldots, N\}.
\end{equation}
In other words, the dual norm of the composite is the composite of the dual norms. 
\end{theorem}

\begin{proof}
We aim at proving that
$$
\sup\left\lbrace \langle A, B\rangle \: : \: B\in\R^{N\times M\times C}, \: \|B\|_{\vec{b}^{*}, a^{*}}\leq 1 \right\rbrace = \|A\|_{\vec{b}, a},
$$
where $\|\cdot\|_{\vec{b}^{*},a^{*}}$ is defined in \eqref{eq:matrixnormdual}. Let $B \in \mathbb{R}^{N\times M\times C}$ satisfying $\|B\|_{\vec{b}^{*},a^{*}}\leq 1$ be fixed but arbitrary, and define $v_i^{\vec{b}} :=  \|A_{i,:,:}\|_{\vec{b}}$ and $v_i^{\vec{b}^{*}}:= \|B_{i,:,:}\|_{\vec{b}^{*}}$ for each $i\in\{1,\ldots, N\}$. Applying H{\"o}lder inequality for both $\|\cdot\|_{\vec{b}}$ and $\|\cdot\|_{a}$ norms yields
\begin{equation*}
\begin{aligned}
\langle A, B \rangle &= \sum_{i=1}^N \sum_{j=1}^{M}\sum_{k=1}^C A_{i,j,k} B_{i,j,k} \leq \sum_{i=1}^N   \|A_{i,:,:}\|_{\vec{b}}  \|B_{i,:,:}\|_{\vec{b}^{*}} =  \langle v^{\vec{b}} , v^{\vec{b}^{*}} \rangle \\
&\leq  \|v^{\vec{b}}\|_a \|v^{\vec{b}^{*}}\|_{a^{*}} = \|A\|_{\vec{b},a} \|B\|_{\vec{b}^{*},a^{*}}\leq \|A\|_{\vec{b},a}.
\end{aligned}
\end{equation*}
The proof reduces now to show that there exists some $B\in\R^{N\times M\times C}$ satisfying $\|B\|_{\vec{b}^{*},a^{*}}\leq 1$ for which the equality holds.

Since $\| \cdot \|_{a^{*}}$ is the dual norm to $\| \cdot \|_{a}$, there exists some $z\in\R^N$, $\| z \|_{a^{*}} \leq 1$, such that $ \langle v^{\vec{b}}, z \rangle = \| v^{\vec{b}}\|_a$. We can additionally assume that $z_i\geq 0$  for all $i\in\{1,\ldots, N\}$. Indeed, suppose that $z_j<0$ for some $j\in\{1,\ldots,N\}$ and define $\widetilde{z}$ with $\widetilde{z}_i=z_i$ for $i\neq j$ and $\widetilde{z}_j=-z_j$. Since $\| z \|_{a^{*}}$ only depends on the absolute values of the $z_i$'s, it follows that $\widetilde{z}$ meets $\|\widetilde{z}\|_{a^{*}}\leq 1$. If $v_j^{\vec{b}}>0$, then one deduces that $\langle v^{\vec{b}},z\rangle  < \langle v^{\vec{b}},\widetilde{z}\rangle$, which contradicts the definition of $\|v^{\vec{b}}\|_a$. If $v_j^{\vec{b}}=0$, then $\|v^{\vec{b}}\|_a=\langle v^{\vec{b}}, z\rangle = \langle v^{\vec{b}}, \widetilde{z}\rangle$ so that we need only to take $\widetilde{z}$ instead of $z$.

On the other hand, since $\| \cdot \|_{\vec{b}^{*}}$ is the dual norm to $\| \cdot \|_{\vec{b}}$, there exists $y^i\in\R^{M\times C}$, $\| y^i \|_{\vec{b}^{*}} \leq 1$, such that $ \langle A_{i,:,:}, y^i \rangle = \| A_{i,:,:}\|_{\vec{b}}$ for all $i\in\{1,\ldots,N\}$. 

Now, it follows from the definitions of $z\in\R^N$ and each $y^i\in\R^{M\times C}$ that
$$
\|A\|_{\vec{b},a} = \|v^{\vec{b}}\|_a = \langle v^{\vec{b}}, z\rangle = \sum_{i=1}^N z_i \|A_{i,:,:}\|_{\vec{b}} =  \sum_{i=1}^N z_i \langle A_{i,:,:}, y^i \rangle =  \sum_{i=1}^N \sum_{j=1}^{M}\sum_{k=1}^C z_i A_{i,j,k} y^i_{j,k},
$$
from where $\|A\|_{\vec{b},a} = \langle A,B \rangle$ by choosing $B_{i,j,k} = z_i y^i_{j,k}$. Furthermore, 
$$
v_i^{\vec{b}^{*}}=\|B_{i,:,:}\|_{\vec{b}^{*}} =  \|z_i y^i\|_{\vec{b}^{*}} = |z_i|\cdot \|y^i\|_{\vec{b}^{*}} \leq |z_i| = z_i, \quad \forall i\in\{1,\ldots, N\}.
$$
Let $v\in\R^N$, $\|v\|_a\leq 1$, be fixed but arbitrary, then
$$
\langle v^{\vec{b}^{*}}, v\rangle = \sum_{i=1}^Nv_i^{\vec{b}^{*}}v_i \leq \sum_{i=1}^N z_i v_i \leq \|z\|_{a^{*}} \|v_i\|_a \leq 1,
$$
which implies $\|B\|_{\vec{b}^{*},a^{*}} = \|v^{\vec{b}^{*}}\|_{a^{*}}=\sup\{ \langle v^{\vec{b}^{*}}, v\rangle : v\in\R^N, \|v\|_a\leq 1\}\leq 1$. This means that we found $B\in\R^{N\times M\times C}$, $\|B\|_{\vec{b}^{*},a^{*}} \leq 1$, such that $\langle A, B \rangle =  \|A\|_{\vec{b},{a}}$, which concludes the proof.
\end{proof}

Theorem \ref{thm:dualnorm} states that the $\ell^{p^{*},q^{*},r^{*}}$ norm is dual to $\ell^{p,q,r}$, where $p^{*}$, $q^{*}$, and $r^{*}$ denote the H{\"o}lder conjugate exponents of $p$, $q$, and $r$, respectively. Similarly, the dual norm of the Schatten $(S^p,\ell^q)$ norm is $(S^{p^{*}},\ell^{q^{*}})$. Furthermore, we have implicitly proved a H{\"o}lder's inequality for collaborative norms.
\begin{lemma}
Under conditions of Theorem \ref{thm:dualnorm}, we have
$$
|\langle A, B\rangle |\leq \|A\|_{\vec{b},a} \cdot \|B\|_{\vec{b}^{*},a^{*}},
$$
for any $A,B\in\R^{N\times M\times C}$.
\end{lemma}

We also furnish ourselves with the subdifferential of the proposed collaborative norms that will be useful for computing their proximal mappings.

\begin{theorem}
\label{thm:subdiffnorm}
Consider $A\in\R^{N\times M\times C}$ and define $v\in\R^N$ such that $v_i:=\|A_{i,:,:}\|_{\vec{b}}$ for each $i\in\{1,\ldots, N\}$. If $\|v\|_{a^{*}}$ only depends on the absolute values of $v_i$'s, then the subdifferential of $\|\cdot\|_{\vec{b},a}$ is given by
\begin{equation}\label{eq:subdiff}
\partial \big( \|A\|_{\vec{b},a}\big) = \left\{ B\in \R^{N\times M\times C} \: :\: \|B\|_{\vec{b}^{*}, a^{*}}\leq 1 \; \mbox{ and } \;\langle B, A\rangle = \| A\|_{\vec{b},a} \right\}.
\end{equation}
\end{theorem}

\begin{proof}
Since  $\|\cdot\|_{\vec{b},a}$ is positively one-homogeneous, it is well-known that its subdifferential is given by
$$
\partial \big(\|A\|_{\vec{b},a}\big) = \left\{ B\in \R^{N\times M\times C}  \:  : \: \langle B, A\rangle = \|A\|_{\vec{b},a}\:\:\text{and} \:\:\langle B, M\rangle \leq \|M\|_{\vec{b},a}\: \forall M\right\}.
$$
Recall now that the Legendre-Fenchel transform of a proper convex function is defined as $f^{*}(y):=\sup_{x} \left\{ \langle y,x\rangle - f(x)\right\}$. Furthermore, the Legendre-Fenchel transform of a norm $f(x):=\|x\|$ turns to be the indicator function on the unit ball of the dual norm:
$$
f^{*}(y)=\left\lbrace \begin{array}{ll} 0 & \text{if } \|y\|_{*}\leq 1, \\ +\infty & \text{otherwise}. \end{array}\right.
$$
We refer the reader to \cite{Hiriart1996, Rockafellar1997} for more details. Therefore, taking the supremum over all $M\in\R^{N\times M \times C}$ in $\langle B, M\rangle \leq \|M\|_{\vec{b},a}$ yields $f^{*}(B)\leq 0$. Due to Theorem \ref{thm:dualnorm}, we necessarily have that $\|B\|_{\vec{b}^{*},a^{*}}\leq 1$, which ends the proof.
\end{proof}

\section{Vectorial TV Revisited}\label{sec:VTV}

The collaborative norms defined in the previous section support most of  pre-existing definitions of TV for vector-valued images, the most relevant of which are displayed in Table \ref{tab:VTVOverview}. For nonlocal TV based models, we refer the reader to our conference paper \cite{DuranMoellerNLTV2015}.

\begin{table}[!htbp]
\centering
\footnotesize
 \begin{tabular}{|c|c|c|}
\hline
  Literature & Continuous Formulation & Collaborative TV \\
  \hline\hline
 \cite{Attouch2006} & $\displaystyle\sum_{k=1}^C \int_\Omega \sqrt{(\partial_x u_k(x))^2 + (\partial_y u_k(x))^2} \, dx$ & $\ell^{2,1,1}(der,pix, col)$\\
   \hline
    $\begin{array}{c} \text{Anisotropic}\vspace{-0.1cm} \\ \text{variant}\end{array}$ & $\displaystyle\sum_{k=1}^C\int_\Omega  \big(|\partial_x u_k(x)| + |\partial_y u_k(x)|\big) \, dx $ & $\ell^{1,1,1}(der,pix, col)$ \\
   \hline\hline
 \cite{Blomgren1998} & $\displaystyle\sqrt{\sum_{k=1}^C \left( \int_\Omega \sqrt{(\partial_x u_k(x))^2 + (\partial_y u_k(x))^2} \, dx \right)^2}$ & $\ell^{2,1,2}(der,pix,col)$ \\
  \hline
       $\begin{array}{c} \text{Anisotropic}\vspace{-0.1cm} \\ \text{variant}\end{array}$ & $\displaystyle\sqrt{\sum_{k=1}^C \left( \int_\Omega ~\big( |\partial_x u_k(x)| + |\partial_y u_k(x)|\big) \, dx \right)^2}$ & $\ell^{1,1,2}(der,pix,col)$ \\
  \hline\hline
 \cite{BressonChan2008, Sapiro1996} & $\displaystyle\int_\Omega \sqrt{\sum_{k=1}^C\Big( \big(\partial_x u_k(x)\big)^2 + \big(\partial_y u_k(x)\big)^2\Big)} \,dx $ & $\ell^{2,2,1}(der,col,pix)$ \\
 \hline
 \multirow{2}{1.84cm}{$ \ $ \vspace{0.25cm}$\begin{array}{c} \text{Anisotropic}\vspace{-0.1cm} \\ \text{variants}\end{array}$}  & $\displaystyle\int_\Omega\left( \sqrt{\sum_{k=1}^C(\partial_x u_k(x))^2} + \sqrt{\sum_{k=1}^C(\partial_y u_k(x))^2} \right) dx$ & $\ell^{2,1,1}(col,der,pix)$ \\
 \cline{2-3} 
 & $\displaystyle\int_\Omega \sqrt{\sum_{k=1}^C \big(|\partial_x u_k(x)| + |\partial_y u_k(x)|\big)^2} \, dx $ & $\ell^{1,2,1}(der,col,pix)$ \\
 \hline\hline
 \multirow{2}{1.48cm}{$\vspace{-0.25cm}$ $\begin{array}{c} \text{Strong}\vspace{-0.1cm} \\ \text{coupling}\end{array}$}  & $\displaystyle\int_\Omega \left( \max_{1\leq k\leq C} |\partial_x u_k(x)| + \max_{1\leq k\leq C}|\partial_y u_k(x)| \right)\, dx$ & $\ell^{\infty,1,1}(col,der,pix)$ \\
 \cline{2-3}
& $\displaystyle\int_\Omega \max_{1\leq k \leq C} \big(\left|\partial_x u_k(x)\right|+ \left|\partial_y u_k(x) \right|\big)\, dx$ & $\ell^{1,\infty,1}(der,col,pix)$ \\
 \hline
   \multirow{2}{1.48cm}{$\vspace{-0.25cm}$ $\begin{array}{c} \text{Isotropic}\vspace{-0.1cm} \\ \text{variants}\end{array}$} & $\displaystyle\int_\Omega \sqrt{\left(\max_{1\leq k\leq C} |\partial_x u_k(x)|\right)^2 + \left(\max_{1\leq k\leq C}|\partial_y u_k(x)|\right)^2} dx$ & $\ell^{\infty,2,1}(col,der,pix)$ \\ 
   \cline{2-3}
& $\displaystyle\int_\Omega \max_{1\leq k \leq C} \sqrt{\left(\partial_x u(x)\right)^2 + \left(\partial_y u(x) \right)^2}\, dx$ & $\ell^{2,\infty1}(der,col,pix)$ \\
\hline
$\begin{array}{c} \text{Supremum}\vspace{-0.1cm} \\ \text{variant}\end{array}$ & $\displaystyle\int_\Omega \left( \max \left\lbrace \max_{1\leq k\leq C} |\partial_x u_k(x)| , \max_{1\leq k\leq C} |\partial_y u_k(x)|\right\rbrace\right)\, dx$ & $\ell^{\infty,\infty,1}(col,der,pix)$\\
 \hline\hline
  \cite{Lefkimmiatis2013, Sapiro1996} &  $\displaystyle\int_{\Omega} \sum_{i=1}^{r}\left|\sigma_i \left(\nabla u(x)\right)\right| \,dx $ & $(S^1(col,der), \ell^{1}(pix))$ \\ \hline
$\begin{array}{c} \text{Frobenius}\vspace{-0.1cm} \\ \text{norm}\end{array}$ &  $\displaystyle\int_{\Omega} \sqrt{\sum_{i=1}^{r}\left(\sigma_i \left(\nabla u(x)\right)\right)^2} \,dx$ & $(S^2(col,der), \ell^{1}(pix))$ \\
 \hline
  \cite{Goldluecke2012, Sapiro1996} & $\displaystyle\int_{\Omega} \max_{1\leq i\leq r} \sigma_i \left(\nabla u(x)\right) \,dx$ & $(S^\infty(col,der), \ell^{1}(pix))$ \\
 \hline
\end{tabular}
\caption{Overview of local vectorial TV approaches and the way they fit in our framework.}
\label{tab:VTVOverview}
\end{table}

Approaches to defining vectorial TV regularizations can roughly be divided into two classes. The first class of methods extend the definition of the scalar case \eqref{scalarTV} to vector-valued images by introducing a suitable channel coupling. The second class of approaches emerges when considering the Riemann geometry of the image manifold. All of them are analyzed below as special cases of CTV. We formally use continuous notations even though our framework is given in the discrete setting.

\subsection{Vectorial TV Models from Channel Coupling}

The first known extension of the total variation to vector-valued images is due to Blomgren and Chan \cite{Blomgren1998}. They applied the Euclidean norm to the vector obtained from the TV contributions across channels, that is,
\begin{equation}\label{eq:VTVBlomgren}
\VTV(\mathbf{u}):= \sqrt{\sum_{k=1}^C \left(\TV(u_k)\right)^2}.
\end{equation}
From the Euler-Lagrange equation associated to \eqref{eq:VTVBlomgren}, one easily observes that there is a global weak channel coupling so that the same per-channel weight is used for all pixels. Consequently, this model favours the restoration of images for which similar noise is measured in each channel. In our framework,  the vectorial TV proposed in \cite{Blomgren1998} can be written as an $\ell^{2,1,2}(der, pix, col)$ penalty.

Probably, the most simple way to introduce multichannel TV is to sum up the contributions of each channel separately \cite{Attouch2006}, which leads to
\begin{equation}\label{eq:VTVAttouch}
\VTV(\mathbf{u}):=\sum_{k=1}^C \TV(u_k).
\end{equation}
Depending on the coupling used along the derivative dimension, one obtains the isotropic version, $\ell^{2,1,1}(der,pix,col)$, which was the one originally proposed in \cite{Attouch2006}, or the anisotropic version, $\ell^{1,1,1}(der,pix,col)$. As pointed out by Goldluecke {\it et al.} \cite{Goldluecke2012}, the drawbacks underlying this approach are color smearing and edge distortion because of the missing channel coupling. We can expect \eqref{eq:VTVAttouch} to be a good choice if there is no particular relation between channels.

In \cite{BressonChan2008}, the isotropic vectorial TV with a local $\ell^2$ channel coupling is proposed:
\begin{equation}\label{eq:VTVBresson}
\VTV(\mathbf{u}) := \int_{\Omega}\sqrt{\sum_{k=1}^C \left(  \left( \partial_x u_k(x)\right)^2 + \left(\partial_y u_k(x) \right)^2 \right)}\,dx,
\end{equation}
which is equivalent to the collaborative norm $\ell^{2,2,1}(der, col, pix)$.
Blomgren and Chan \cite{Blomgren1998} noted that this method actually favours gray-value images over colored ones, which leads to color smearing in denoising applications. 

The inclusion of additional color transforms has been proposed to improve the performance of vectorial TV methods. It is well known that RGB channels of natural images are highly correlated. In view of this, some researchers incorporated different color transforms into the definition of vectorial TV and penalized the gradient of each component in the new basis separately \cite{ChanKang2001, Condat2012}:
\begin{equation}\label{eq:VTVCondat}
\VTV(\mathbf{u}):= \sum_{k=1}^C \TV(u_k\circ \psi),
\end{equation}
where $\psi:R^C\rightarrow\R^{\widetilde{C}}$ is an orthonormal transform between color spaces. The key idea is to choose $\psi$ such that it provides effective reduction of the correlation among channels. Note that \eqref{eq:VTVCondat} is equivalent to penalize the collaborative norm $\ell^{1,1,1} (der, pix, \psi(col))$ for the anisotropic variant and $\ell^{2,1,1}(der, pix, \psi(col))$ for the isotropic variant.


\subsection{Vectorial TV Models from Riemann Geometry}

A color image can be considered as a parametric two-dimensional manifold embedded in a $C-$dimensional space \cite{DiZenzo1986}. In this framework, the metric tensor of the manifold is analogous to the structure tensor of the image, that is, $\left(\nabla\mathbf{u}\right)^{\top}\nabla\mathbf{u}$. Therefore, the eigenvectors of  $\left(\nabla\mathbf{u}\right)^{\top}\nabla\mathbf{u}$ determine the directions of maximal and minimal change and the eigenvalues, which will be respectively denoted by $\lambda^{+}$ and $\lambda^{-}$, give their rate of change.

In this setting, Sapiro \cite{Sapiro1996} introduced the following general vectorial TV model:
\begin{equation}\label{eq:VTVSapiro}
\VTV(\mathbf{u}):=\int_{\Sigma} f(\lambda^{+}, \lambda^{-})\, dx,
\end{equation}
where $\Sigma$ denotes the image manifold and $f$ is a suitable scalar-valued function. In general, \eqref{eq:VTVSapiro} is defined for differentiable functions, but only for special cases one has dual formulations that extend it to locally integrable functions. This is the case of the Frobenius norm of the gradient given by
\begin{equation}\label{eq:VTVFrobenius}
\VTV(\mathbf{u}):=\int_{\Omega} \|\nabla \mathbf{u}(x)\|_F \,dx,
\end{equation}
which follows from \eqref{eq:VTVSapiro} by considering $f(\lambda^{+}, \lambda^{-})=\sqrt{\lambda^{+}+\lambda^{-}}$. Note that \eqref{eq:VTVFrobenius} is equal to the definition of the vectorial TV given in \eqref{eq:VTVBresson} and, thus, either $\ell^{2,2,1}(der, col, pix)$ or $(S^2(col, der),\ell^1(pix))$ can be used in our framework.

Based on  the class of methods presented by Sapiro, Goldluecke {\it et al.} \cite{Goldluecke2012} showed that the natural choice for vectorial TV arising from geometric measure theory is to penalize the largest singular value of the Jacobian:
\begin{equation}\label{eq:VTVGoldluecke}
\VTV(\mathbf{u}) = \int_{\Omega} \sigma_1(\nabla\mathbf{u})\, dx,
\end{equation}
where $\sigma_1$ is the largest singular value of $\nabla\mathbf{u}$ or, equivalently, the largest eigenvalue of the structure tensor $\left(\nabla\mathbf{u}\right)^{\top}\nabla\mathbf{u}$. The regularization introduced in \cite{Goldluecke2012} is known as the spectral or Schatten$-\infty$ norm and fits in our framework as $(S^{\infty}(col, der), \ell^1(pix))$.

Recently, Holt \cite{Holt2014} interpreted \eqref{eq:VTVGoldluecke} as a special case of spatially-local coupling models. The author proposed to smoothen a differentiable function $\mathbf{u}$ by penalizing its Jacobian matrix:
\begin{equation}\label{eq:VTVHolt}
\VTV(\mathbf{u}):= \int_{\Omega} \phi\left(J_{\mathbf{u}}(x)\right) dx,
\end{equation}
where $J_{\mathbf{u}}:\Omega\rightarrow\R^{M\times C}$ denotes the Jacobian matrix of $\mathbf{u}$, so $\left[J_{\mathbf{u}}(x)\right]_{j,k}=\frac{\partial}{\partial x_j}u_k(x)$. This Jacobian framework is closely related to \eqref{eq:VTVSapiro}, since the structure tensor is given by $J_{\mathbf{u}}(x) J^{\top}_{\mathbf{u}}(x)$ at each point in the image. Note that \eqref{eq:VTVAttouch} and \eqref{eq:VTVBresson} are special cases of \eqref{eq:VTVHolt}, however, any method using spatially-global coupling such as \eqref{eq:VTVBlomgren} is not covered by Holt's approach. In \cite{Holt2014}, the author considered only functions that are written in terms of the singular values of $J_{\mathbf{u}}$. Therefore, the Frobenius norm \eqref{eq:VTVFrobenius} follows from $\phi:=\sqrt{\lambda^{+} + \lambda^{-}}$, the spectral norm \eqref{eq:VTVGoldluecke} follows from $\phi:=\sqrt{\lambda^{+}}$, and the nuclear norm \cite{Lefkimmiatis2013} follows from $\phi:=\sqrt{\lambda^{+}} + \sqrt{\lambda^{-}}$. In our framework, the regularizations arising from \eqref{eq:VTVHolt} are given by $(S^p(col, der), \ell^1(pix))$.

Another relevant approach based on Riemann geometry was pioneered in the framework by Kimmel, Malladi, and Sochen \cite{KimmelMalladi2000, Sochen1998}, who considered the graph of an image embedded in a $(C+2)-$dimensional space and proposed an area minimizing flow. This class of regularizations lead to diffusion equations with the direction given by the Beltrami flow. Roussos and Maragos \cite{Roussos2010} generalized the Beltrami flow by using higher dimensional mappings which depend on image patches:
\begin{equation}\label{eq:VTVBeltrami}
\VTV(\mathbf{u}):= \int_{\Omega} \psi\left( \lambda_r^{+} , \lambda_r^{-}\right)\, dx,
\end{equation}
where $\psi$ is increasing with respect to both arguments, and $\lambda^{+}_r$ and $\lambda^{-}_r$ are the larger and smaller eigenvalues of the structure tensor $K_r \ast \left(\nabla\mathbf{u}\right)^{\top}\nabla\mathbf{u}$, with $K_r$ being a non-negative, rotationally symmetric convolution kernel. In posterior works \cite{LefkimmiatisRoussos2015, Lefkimmiatis2013}, a deeper analysis for the particular choice $\psi(\lambda^{+}_r, \lambda^{-}_r) = \big\|\big( \sqrt{\lambda^{+}_r}, \sqrt{\lambda^{-}_r}\big)\big\|_p$ was developed. There, the {\it tensor} TV that arises when $p=1$ is renamed as the {\it nuclear norm}, which can be written as $(S^1(col, der), \ell^1(pix))$ in our case. In order to incorporate information from the vicinity of every point in the image domain as in \eqref{eq:VTVBeltrami}, we only have to incorporate the nonlocal gradient operator (see \cite{DuranMoellerNLTV2015} for more details) and penalize the resulting structure with the help of the $(S^p(col, der), \ell^1(pix))$ norm. Contrary to our approach, neither spatially-global coupling norms like \eqref{eq:VTVBlomgren} nor TV with $\ell^{\infty}$ channel coupling are covered by \eqref{eq:VTVBeltrami}.

\subsection{Other Vectorial TV Models}

For the sake of completeness, we should also mention that there exist several further TV variants, such as nonconvex regularizations based on $\ell^p$ norms with $p<1$ \cite{Krishnan2009}, and nonconvex extensions for minimizing the rank of submatrices in a TGV framework \cite{Mollenhoff2015}. Additional work has been done on improving TV with the help of Bregman iteration \cite{MoellerBrinkmann2014, Osher2005}. The study of the previously mentioned classes of methods, however, goes beyond the scope of this paper.

\section{Which Channel Coupling Disfavours Color Artifacts?}\label{sec:singVectors}

For discussing the question which CTV methods work better, we have to understand what kind of properties they try to impose on the reconstructed image. In this section, we analyze the differences between a color coupling in the $\ell^1$, $\ell^2$ and $\ell^\infty$ fashion with the help of the generalized concept of singular vectors \cite{BenningBurger2013}. The question whether a strong or a weak coupling leads to better results depends on the type of correlation in the data. Our investigation explains why the $\ell^\infty$ norm leads to the strongest relation, makes the most prior assumptions and has the greatest potential to reduce color artifacts.

Benning and Burger developed in \cite{BenningBurger2013} a generalization of the concept of singular vectors and singular values for arbitrary convex regularizations, and showed that a signal can be restored particularly well if it is a singular vector to the regularization for which is used. The authors also showed that, even in the case of noisy data, an exact reconstruction (up to a loss of contrast) is possible under certain conditions. In this sense, they provided a theoretical basis for explaining that TV regularization works particularly well for piecewise constant images.

In order to analyze the behaviour of collaborative norms for vectorial TV, we restrict ourselves to the case of image denoising modelled by anisotropic vectorial TV, namely comparing $\ell^{1,1,1}(col,der,pix)$, $\ell^{2,1,1}(col,der,pix)$ and $\ell^{\infty,1,1}(col,der,pix)$ norms. Let $D$ denote the usual local discrete gradient operator such that $D\mathbf{u}\in Y$, and consider only the energy due to the regularization, that is, $J(\mathbf{u}):=F(D\mathbf{u})= \|D\mathbf{u}\|_{\vec{b},a}$. We fix $M=2$ since only $x-$ and $y-$derivatives are considered. Furthermore, let $(x,y)$ denote any pixel of the image, with $x$ being the row and $y$ the column in the rectangular domain. We provide below the definitions of singular value and singular vector for image denoising problems.

\begin{definition}
Let $J$ be a convex functional with $\partial J(\mathbf{u})\neq \emptyset$ at every $\mathbf{u}\in\dom J$. Then, every function $\mathbf{u}_{\lambda}\in X$ satisfying $\|\mathbf{u}_{\lambda}\|=1$ and $\lambda \mathbf{u}_\lambda \in \partial J(\mathbf{u}_\lambda)$ is called a \emph{singular vector} of $J$ with corresponding \emph{singular value} $\lambda$.
\end{definition}

In the case of $J$ being one-homogeneous we even have that $\lambda \mathbf{u}_{\lambda}\in \partial J(\mathbf{u}_{\lambda})$ is equivalent to $\lambda=J(\mathbf{u}_{\lambda})$, which easily follows from Euler's identity \cite{YangWei2008}: $\langle \mathbf{v}, \mathbf{u}\rangle = J(\mathbf{u})$ for any $\mathbf{v}\in \partial J(\mathbf{u})$. Note that, for any $\mathbf{u} \in \partial J(\mathbf{u})$, one can define $\lambda := \sqrt{J(\mathbf{u})}=\|\mathbf{u}\|$ and $\mathbf{u}_\lambda := \frac{1}{\lambda}\mathbf{u}$ such that $\mathbf{u}_{\lambda}$ is a singular vector. We will therefore omit $\lambda$ and focus on the construction of $\mathbf{u}\in X$ satisfying $\mathbf{u} \in \partial  J(\mathbf{u})$. Since $J(\mathbf{u})=\|D\mathbf{u}\|_{\vec{b},a}$, the latter condition is met if $\mathbf{u}$ can be written as $\mathbf{u} = D^T \mathbf{z}$ for some $\mathbf{z}\in \partial_{D\mathbf{u}} \big(\|D\mathbf{u}\|_{\vec{b},a}\big)$
which, by applying Theorem \ref{thm:subdiffnorm}, is equivalent to
\begin{equation} \label{eq:subdiff1}
\langle \mathbf{z}, D\mathbf{u} \rangle = \|D \mathbf{u}\|_{\vec{b},a} \:\: \text{and}\:\: \|\mathbf{z}\|_{\vec{b}^{*}, a^{*}}\leq 1.
\end{equation}
 
In what follows, all mathematical proofs have been moved to Appendix \ref{app:singVectors}. We aim at finding some $\mathbf{z}\in Y$ satisfying \eqref{eq:subdiff1} to determine $\mathbf{u}=D^T\mathbf{z} \in \partial J(\mathbf{u})$. Motivated by \cite{BenningBurger2013}, it makes sense to consider piecewise linear funtions whose changes happen only at  $\{-1,+1\}$. More specifically, we choose
\begin{equation} \label{eq:zEquation}
z^1_k(x,y) = c_k^1 l_k^1(x)\:\: \text{and}\:\: z^2_k(x,y) = c_k^2 l_k^2(y),\quad \forall k\in\{1,\ldots, C\},
\end{equation}
for some $c_k^r \in \mathbb{R}$, and $l_k^r$ having the following properties: $|l_k^r(x)| \leq 1$ for all $x$, $l_k^r$ piecewise linear, and the linearity changing at $x$ only if $|l_k^r(x)| =1$. The details for why these functions have to look like this are left for the proof in Appendix \ref{app:singVectors}. We simply illustrate examples for $z^1_k$ and $z^2_k$ in Figure \ref{fig:zExamples}. It is remarkable that singular vectors to the CTV methods under consideration can all be written in the form of \eqref{eq:zEquation} and only differ in two aspects. First, the $\ell^{1}$ case allows different $l^r_k$ for different color channels, while the $\ell^{2}$ and $\ell^{\infty}$ norms do not. Second, the coefficients $c^r_k$ are different for each regularization.

\begin{figure}[!htpb]
\centering
\begin{tabular}{cc}
   \includegraphics[scale=0.35]{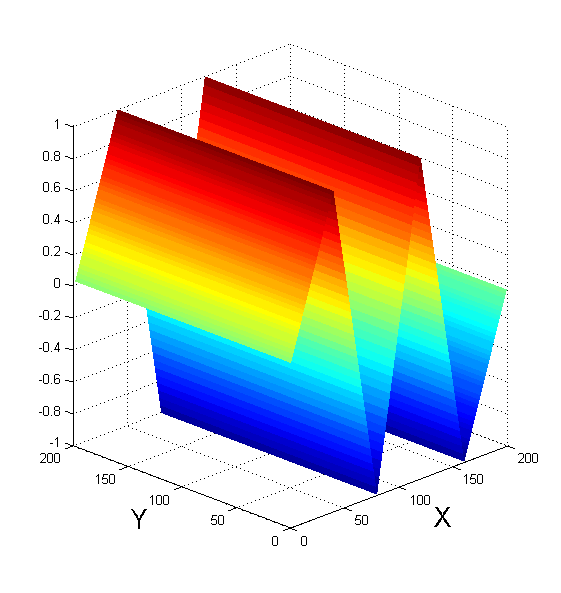} &
   \includegraphics[scale=0.35]{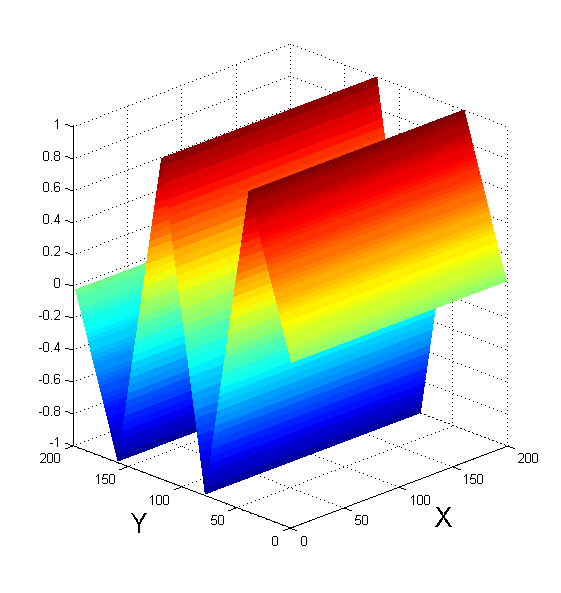} \\
   Illustration of $z_k^1$ & Illustration of $z_k^2$ \\ 
\end{tabular}
\caption{Examples for functions $z_k^1$ and $z_k^2$. As we can see, they are constant in one direction, piecewise linear in the other, and the points where the piecewise linearity changes are at $\{-1,+1\}$.}
\label{fig:zExamples}
\end{figure}

Table \ref{tab:singularVectors} shows the precise construction of singular vectors. The results displayed there meet what we would expect based on the regularization behavior of the different methods. For the $\ell^{1}$ case, each channel can have its own $l^r_k$ such that jumps can be at different positions in the different channels. Since no relation on the positions of the jumps in different channels is imposed, we can expect the $\ell^{1,1,1}$ norm to not suppress color artifacts and not change the position of the edges. This is a theoretical explanation for what we saw in Figure \ref{fig:SyntheticEdge}. Both $\ell^{2}$ and $\ell^{\infty}$ couplings require the $l^r_k$ to be independent of $k$, that is, jumps in different color channels are encouraged to be at the same position. The difference between them is that the size of the jumps, corresponding to the coefficients $c^r_k$, are allowed to be arbitrary in the $\ell^{2}$ case, while they have to be either zero or of equal magnitude in the $\ell^{\infty}$ norm. Equal magnitude of the jumps in all three color channels leads to a grayscale image. This tells us that the regularization based on $\ell^{\infty,1,1}$, opposed to $\ell^{2,1,1}$, encourages jumps that occur in all three channels to only change the intensity but not the color of the image. Looking at the results in Figure \ref{fig:SyntheticEdge}, we can see again that the singular vector analysis confirms exactly what we observed in practice. 

\begin{table}[!htpb]
\footnotesize
\begin{center}
\newcolumntype{C}[1]{>{\centering\let\newline\\\arraybackslash\hspace{0pt}}m{#1}}
\renewcommand{\arraystretch}{2}
\begin{tabular}{|c|c|C{4.5cm}|}
\hline
Regularization & Singular Vectors & Properties \\
\hline\hline
$\|D\mathbf{u}\|_{1,1,1}$  & $u_k(x,y) = -c_k^1 ~ D_x l_k^1(x) - c_k^2 ~ D_y l_k^2(y)$ & $c^r_k \in \{0, \pm 1\}$\\
\hline

$\|D\mathbf{u}\|_{2,1,1}$  & $u_k(x,y) = -c_k^1 ~ D_xl^1(x) - c_k^2 ~ D_yl^2(y)$ & The piecewise linear functions $l^r$ do not depend on $k$, $\|c^r\|_2 =1$ \\
\hline

$\|D\mathbf{u}\|_{\infty,1,1}$  & $u_k(x,y) = -c_k^1 ~ D_xl^1(x) - c_k^2 ~ D_yl^2(y)$ & The piecewise linear functions $l^r$ do not depend on $k$, $c^r_k \in \{0, \pm 1\}$  \\
\hline
\end{tabular}
\end{center}
\caption{Comparison of singular vectors for coupling the color channels in an $\ell^1$, $\ell^2$ and $\ell^\infty$ fashion. In this setting, $D_x$ and $D_y$ denote the differences in the horizontal and vertical directions, respectively.}
\label{tab:singularVectors}
\end{table}

For illustration purposes, Figure \ref{fig:singularVectors} shows some examples of singular vectors. Depending on the type of jumps in the data, that is, jumps in different color channels being independent of one another, jumps being at the same position but changing the color, or jumps being at the same position and likely not changing the color, the $\ell^{1,1,1}$, the $\ell^{2,1,1}$ or the $\ell^{\infty,1,1}$ norms will show a superior performance. Interestingly, our numerical results in Section \ref{sec:results} indicate that a suppression of color artifacts by using $\ell^{\infty,1,1}$ is more important than making weaker and more general assumptions on the types of jumps in natural images. 

\begin{figure}[!htpb]
\centering
\begin{tabular}{p{3.92cm} p{3.92cm} p{3.92cm}}
 	\includegraphics[scale=0.56]{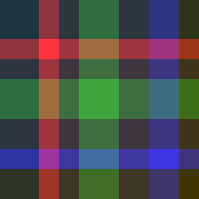} &
 	\includegraphics[scale=0.56]{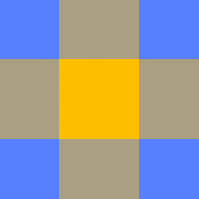} &
 	\includegraphics[scale=0.56]{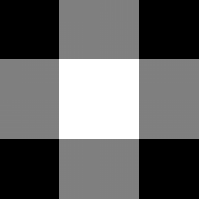} \\
        Image generated from $z_k^r$ which are different for all $k$, resulting in an image where all three color channels have jumps at different positions. Up to a scaling, this is a singular vector to $\ell^{1,1,1}$ but not to $\ell^{2,1,1}$ or $\ell^{\infty,1,1}$. & 
        Image for which the jumps are at the same positions (all $z_k^r$ are equal), but at which the coefficients $c^r_k$ are $[0.54, 0.2, -0.82]$. Up to a scaling, this is a singular vector to $\ell^{2,1,1}$ but not to $\ell^{1,1,1}$  or $\ell^{\infty,1,1}$. \vspace{0.2cm} &
        Image with jumps at the same positions and with all coefficients being equal to one. Up to a scaling, this is a singular vector to $\ell^{1,1,1}$, $\ell^{2,1,1}$, and $\ell^{\infty,1,1}$. \\ 
 	\includegraphics[scale=0.56]{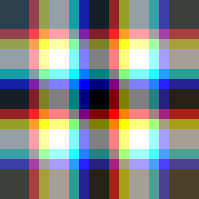} &
 	\includegraphics[scale=0.56]{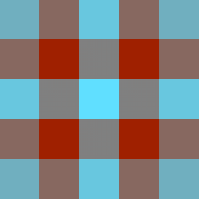} &
 	\includegraphics[scale=0.56]{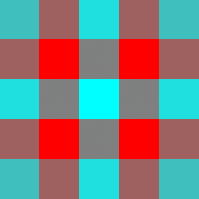} \\
        Image with four instead of two jumps in each channel and in each direction. Again, the edge sets of the different color channels are different. Up to a scaling, this is a singular vector to $\ell^{1,1,1}$ but not to $\ell^{2,1,1}$ or $\ell^{\infty,1,1}$. &
        Image with four jumps in each direction, jumps in the different color channels being aligned, and color channel coefficients being $[0.2, -0.59, -0.78]$. Up to a scaling, this is a singular vector to $\ell^{2,1,1}$ but not to $\ell^{1,1,1}$ or $\ell^{\infty,1,1}$. &
         Image with four jumps in each direction, jumps in the different color channels being aligned, and color channel coefficients being $[1,-1,-1]$. Up to a scaling, this is a singular vector to $\ell^{1,1,1}$, $\ell^{2,1,1}$, and $\ell^{\infty,1,1}$. \\ 
\end{tabular}
\caption{Illustrating singular vectors to vectorial TV regularization using collaborative norms. While the true singular vectors have zero mean, the above images have been rescaled to lie between zero and one for visualization purposes.}
\label{fig:singularVectors}
\end{figure}

\section{Numerical Minimization}\label{sec:numerics}

It is remarkable that all variants of different norms imposed on the three-dimensional structure can be solved very efficiently by using splitting techniques. The only thing that changes when changing the regularization is the proximity operator, which is discussed below.

Recall that the proximity operator of a proper, convex, and l.s.c. function $f$ is
\begin{equation}\label{def:proxop}
\prox_{\tau f} (x) = \arg \min_{y} \left\lbrace \dfrac{1}{2} \|y-x\|^2 + \tau f(y)\right\rbrace,
\end{equation}
where $\alpha>0$ is a scalar parameter. Furthermore, {\it Moreau's identity} connects the proximity operator and its Legendre-Fenchel transform in the following way: 
\begin{equation}\label{eq:Moreauidentity}
x = \prox_{\tau f}(x) + \tau \prox_{\frac{1}{\tau}f^{*}}\left(\dfrac{x}{\tau}\right).
\end{equation}

\subsection{Proximal Map of CTV Regularizers}
Theorem \ref{thm:subdiffnorm} allows us to write the optimality condition to \eqref{def:proxop}  as 
\begin{equation}\label{eq:proxformula}
\widehat{A}=\prox_{\tau\|\cdot\|_{\vec{b},a}}(A)\: \Leftrightarrow\: \|A-\widehat{A}\|_{\vec{b}^{*}, a^{*}}\leq \tau\:\: \text{and} \:\:\langle A, \widehat{A}\rangle = \tau\|\widehat{A}\|_{\vec{b},a} + \|\widehat{A}\|^2_2,
\end{equation}
for any $A\in Y$. In this setting, $\|\cdot\|_2$ denotes the Euclidean norm applied to the vectorial structure obtained by rearranging the original three-dimensional matrix into a vector. When it is not possible to obtain an explicit solution from \eqref{eq:proxformula}, one usually invokes duality through Moureau's identity \eqref{eq:Moreauidentity}:
\begin{equation}\label{eq:proxformuladual}
\widehat{A} = A - \proj_{\frac{1}{\tau}\|\cdot\|_{\vec{b}^{*},a^{*}}\leq 1}(A),
\end{equation}
where $\proj_{\frac{1}{\tau}\|\cdot\|_{\vec{b}^{*},a^{*}}\leq 1}$ denotes the projection operator onto the dual ball of radius $\tau$.

\begin{example}
We now display the proximal mappings of the regularizations based on $\ell^{p,q,r}$ norms which will be used in the experimental section.

\begin{itemize}
\item The proximity operator of the $\ell^{1,1,1}$ norm decouples in all variables and each problem just contains an absolute value penalty:
$$
 \left(\prox_{\tau\|\cdot\|_{1,1,1}}(A) \right)_{i,j,k} = \max \left( |A_{i,j,k}|-\tau, 0\right) \sign\big(A_{i,j,k}\big).
$$

\item By a short computation, one obtains the proximal mapping of the $\ell^{2,1,1}$ norm as the (generalized) shrinkage:
$$
\left(\prox_{\tau\|\cdot\|_{2,1,1}}(A) \right)_{i,j,k} = \max \left( \|A_{i,j,:}\|_2-\tau, 0\right)\dfrac{A_{i,j,k}}{\|A_{i,j,:}\|_2},
$$
as well the proximal mapping of the $\ell^{2,2,1}$ norm:
$$
\left(\prox_{\tau\|\cdot\|_{2,2,1}}(A) \right)_{i,j,k} = \max \left( \|A_{i,:,:}\|_{2,2}-\tau, 0\right)\dfrac{A_{i,j,k}}{\|A_{i,:,:}\|_{2,2}}.
$$

\item Whenever the supremum norm is involved, it is more convenient to use \eqref{eq:proxformuladual} to express the proximity operator by the proximity operator of its dual. For the $\ell^{\infty,1,1}$ norm, one has 
$$
\left(\prox_{\tau\|\cdot\|_{\infty,1,1}}(A) \right)_{i,j,k} = A_{i,j,k} -\tau \sign\big(A_{i,j,k}\big)\left(\proj_{\|\cdot\|_1\leq 1} \left(\dfrac{1}{\tau}|A_{i,j,:}| \right)\right)_{i,j,k},
$$
where $|A_{i,j,:}|$ denotes the component-wise absolute value of vector $A_{i,j:}$ and $\proj_{\|\cdot\|_1\leq 1}$, the projection onto the unit $\ell^1$ norm ball. Similarly, we obtain the proximity operator of the $\ell^{\infty,\infty,1}$ norm as
$$
\left(\prox_{\tau\|\cdot\|_{\infty,\infty,1}}(A) \right)_{i,j,k} = A_{i,j,k} -\tau \sign\left(A_{i,j,k}\right)\hspace{-0.05cm}\left(\proj_{\|\cdot\|_{1,1}\leq 1} \left(\dfrac{1}{\tau}|A_{i,:,:}| \right)\right)_{i,j,k},
$$
with $\proj_{\|\cdot\|_{1,1}\leq 1}$ denoting the projection operator onto the unit $\ell^{1,1}$ ball. Finally, the proximity operator of the $\ell^{\infty,2,1}$ norm is
$$
\left(\prox_{\tau\|\cdot\|_{\infty,2,1}}(A) \right)_{i,j,k} = A_{i,j,k} -\tau \sign\big(A_{i,j,k}\big)\left(\proj_{\|\cdot\|_{1,2}\leq 1} \left(\dfrac{1}{\tau}|A_{i,:,:}| \right)\right)_{i,j,k},
$$
where $\proj_{\|\cdot\|_{1,2}\leq 1}$ denotes the projection operator onto the unit $\ell^{1,2}$ ball. 
\end{itemize}
\end{example}

Let us now discuss the proximity operator of the $\ell^{2,\infty,1}$ norm. For that purpose, we require a previous result that states the chain rule for subdifferentials. The proof is outlined in Appendix \ref{app:ThmSubdif}.

\begin{theorem}[Chain Rule of Subdifferentials]\label{th:chainRule}
Let $f:\R^n\rightarrow \R^m$ be a vector-valued function such that $f_j:\R^n \rightarrow \R$ is proper and convex for each $j\in\{1,\ldots, m\}$. Let $g:\R^m \rightarrow \Rinf$ be convex, proper and nondecreasing in each argument. Then,
\begin{equation}\label{eq:chainRule1}
\partial(g \circ f)(x_0) \supseteq \left\lbrace \xi \in \R^n \: : \: \xi = \sum_{j=1}^m q_j v_j, \begin{array}{c} q=(q_1, \ldots, q_m)\in\partial g(f(x_0)), \vspace{0.1cm}\\ v_j \in \partial f_j(x_0), \,\,\forall 1 \leq j \leq m\end{array}\right\rbrace
\end{equation}
at any $x_0\in \dom (g\circ f)$. If further $x_0\in \inter \dom (g\circ f)$ and all $f_j$ are locally l.s.c., then the inclusion in \eqref{eq:chainRule1} becomes an equality.
\end{theorem}

Although it seems to be difficult to compute the proximal mapping of the $\ell^{2,\infty,1}$ norm at first glance,  Theorem \ref{th:chainRule} leads to a particularly interesting observation regarding functionals having an $\ell^2$ norm as an inner regularization. The following result will provide us the key for computing this class of proximal operators.

\begin{theorem}
\label{thm:compositionL2Prox}
Let $f:\R^{n\times m} \rightarrow \R^n$ be defined  as
$$
f_i(u) := \|u_{i,:}\|_2 = \sqrt{\sum_{j=1}^m u_{i,j}^2}, \quad \forall i\in\{1,\ldots, n\}, \quad \forall u \in \R^{n\times m},
$$
and let $g: \R^n \rightarrow \R$ be any proper, convex function that is nondecreasing in each argument. Then, the proximity operator of $g\circ f$ is
$$
\left(\prox_{\tau(g \circ f)}(u)\right)_{i,j} = \frac{u_{i,j}}{\|u_{i,:}\|_2} ~ \max(\|u_{i,:}\|_2 - \tau v_i,0), \quad \forall i,j,
$$
where $v_i$, $i\in\{1,\ldots, n\}$, are the components of the vector $v\in\R^{n}$ given by
\begin{equation} \label{eq:sFormula}
v = \arg\min_{w\in\R^n} \dfrac{1}{2} \left\| w - \dfrac{1}{\tau}f(u) \right\|^2_2 + \frac{1}{\tau} g^{*}(w).
\end{equation}
\end{theorem}

\begin{proof}
The optimality condition arising from \eqref{def:proxop} yields
\begin{equation}\label{eq:sol}
 \prox_{\tau (g\circ f)}(u)= u - \tau \xi, \:\:\: \text{for some}\:\:  \xi \in \partial (g \circ f)(\widehat{u}).
\end{equation}
Let us define $\widehat{u}$ as
\begin{equation}\label{eq:compositionL2Prox}
\widehat{u}_{i,j}= \frac{u_{i,j}}{\|u_{i,:}\|_2} ~ \max(\|u_{i,:}\|_2 - \tau v_i,0), \quad \forall i,j,
\end{equation}
where $v_i$, $i\in\{1,\ldots, n\}$, are the components of the vector $v\in\R^{n}$ solving \eqref{eq:sFormula}. We aim at proving that $\widehat{u}$ satisfies \eqref{eq:sol}. For that purpose, note that $\widehat{u}$ in \eqref{eq:compositionL2Prox} can be stated as the solution of a weighted $\ell^{1,2}$ regularized problem:
\begin{equation}\label{eq:sol1}
\widehat{u}_{i,:} = u_{i,:} - \tau v_i  z_i, \:\:\: \text{for some}\:\:  z_i \in \partial \left( \|\widehat{u}_{i,:}\|_2\right), \quad \forall i\in\{1, \ldots, n\}.
\end{equation}
In view of \eqref{eq:sol} and \eqref{eq:sol1}, we only need to prove that the matrix with columns $v_iz_i$ is in $\partial (g \circ f)(\widehat{u})$. Due to the chain rule stated in Theorem \ref{th:chainRule}, this follows whenever $z_i \in \partial f_i(\widehat{u})$, which is true by definition of $z_i$, and $v \in \partial g(f(\widehat{u}))$. In order to prove the latter, note that \eqref{eq:sFormula} yields the optimality condition 
\begin{equation}
\label{eq:tDefinition}
t = f(u) - \tau v,
\end{equation} 
for some $t \in \partial g^*(v)$ or, equivalently, $v \in \partial g(t)$. It is thus sufficient to show that $t_i = \|\widehat{u}_{i,:}\|_2\equiv f_i(\widehat{u})$. From \eqref{eq:tDefinition}, we see that $t_i = \|u_{i,:}\|_2 - \tau v_i$ for each $i\in\{1,\ldots, n\}$. Since $g$ is nondecreasing in each argument, then its proximity operator is nonnegativity preserving and so $t_i \geq 0$. Consequently, if $\|\widehat{u}_{i,:}\|_2=0$, then \eqref{eq:compositionL2Prox} implies $ \|u_{i,:}\|_2 \leq \tau v_i$ and, thus,  $t_i=0$. Otherwise, it follows that
$$
\|\widehat{u}_{i,:}\|_2 = \left\| \dfrac{u_{i,:}}{\|u_{i,:}\|_2} (\|u_{i,:}\|_2 - \tau v_i) \right\|_2 = \big|\|u_{i,:}\|_2 - \tau v_i\big| = t_i, $$
which completes the proof.
\end{proof}

\begin{example}
By Theorem \ref{thm:compositionL2Prox}, the proximity operator of the $\ell^{2,\infty,1}$ norm is
$$
\left(\prox_{\tau\|\cdot\|_{2,\infty,1}}(A)\right)_{i,j,k} =\dfrac{A_{i,j,k}}{\|A_{i,j,:}\|_2} ~ \max\left(\|A_{i,j,:}\|_2 - \tau v_{i,j},0\right), \\
$$
where
$$
v_{i,j} = \left(\proj_{\| \cdot \|_1 \leq 1}\left(\dfrac{1}{\tau}\big(\|A_{i,j,:}\|_2\big)_j\right)\right)_{i,j}.
$$
In the above formula, $(\|A_{i,j,:}\|_2)_j$ denotes the vector we obtain by stacking $\|A_{i,j,:}\|_2$ for all $j\in\{1,\ldots, M\}$. Note that the proximal mappings associated to $\ell^{2,1,1}$ and $\ell^{2,2,1}$ can also be computed by means of Theorem \ref{thm:compositionL2Prox}.
\end{example}

Finally, the proximal mapping associated to the $(S^p,\ell^q)$ norm is a simple combination of a singular value decomposition followed by the proximity operator of the corresponding $\ell^p$ norm. Since the regularizations considered in this paper, which base on $(S^1, \ell^1)$ and $(S^{\infty}, \ell^1)$ norms, have an outer $\ell^1$ norm, then the computation of their proximity operators decouples at each pixel. By denoting $B:=A_{i,:,:}^{\top}$, we are thus left with a problem of the form
\begin{equation}\label{eq:proxSp_1}
\min_{D\in\R^{C\times M}}\dfrac{1}{2}\|D-B\|_2^2 + \tau\|D\|_{S^p},
\end{equation}
the solution of which is given in the following proposition.

\begin{proposition}\label{prop:proxSp}
Let $U\Sigma_0 V^T$ be the singular value decomposition of a matrix $B\in\R^{C\times M}$. Then, the proximity operator of the $S^p-$norm is given by
$$
\prox_{\tau S^p} (B) = BV\widehat{\Sigma}\Sigma_0^{\dag}V^{\top},
$$
where $\Sigma_0^{\dag}$ denotes the pseudo-inverse matrix of $\Sigma_0$ and $\widehat{\Sigma}=\prox_{\tau\|\cdot\|_{p}}(\diag(\Sigma_0))$.
\end{proposition}

In the following example, we show the proximal mappings of the CTV regularizations using the Schatten norms we are interested in.

\begin{example}
For $A\in\R^{N\times M\times C}$, the proximity operator of the $(S^p,\ell^1)$ norm is
$$
\left(\prox_{\tau (S^p,\ell^1)}(A)\right)_{i} = A_{i,:,:}^{\top}V_i\big(\prox_{\tau\|\cdot\|_p}(\diag(\Sigma_i))\big)\Sigma_i^{\dag}V_i^{\top},\quad\forall i\in\{1,\ldots, N\},
$$
where $A_{i,:,:}^{\top}=U_i\Sigma_iV_i^{\top}$ is the singular value decomposition of $A_{i,:,:}^{\top}\in\R^{C\times M}$.
\end{example}

\subsection{Solving the Minimization Problem}

For solving the optimization problem \eqref{eq:minproblem} that arises from the proposed collaborative TV, we use the primal-dual hybrid gradient (PDHG) method \cite{ChambollePock2011, EsserZhang2010, GoldsteinEsser2013, Zhu2008}, a powerful optimization algorithm that breaks complex problems into simple sub-steps and can handle non-smoothness of the energy functional.

By introducing an auxiliary variable $\mathbf{g}\in Y$ and the constraint $K\mathbf{u}=\mathbf{g}$ in \eqref{eq:minproblem}, then we obtain the following formulation of the original problem:
$$
\min_{\mathbf{u}\in X, \, \mathbf{g}\in Y}  G(\mathbf{u}) + \|\mathbf{g}\|_{\vec{b},a} \quad \text{subject to} \quad K\mathbf{u}=\mathbf{g}.
$$
Now, consider the Lagrangian $L\left(\mathbf{u},\mathbf{g},\mathbf{q}\right) = G(\mathbf{u}) + \|\mathbf{g}\|_{\vec{b},a} + \langle \mathbf{q}, K\mathbf{u}-\mathbf{g}\rangle_Y$, then the associated primal-dual problem is
\begin{equation}\label{eq:primaldualproblem2}
\max_{\mathbf{q}\in Y}\min_{\mathbf{u}\in X,\, \mathbf{g}\in Y} G(\mathbf{u}) + \|\mathbf{g}\|_{\vec{b},a} +  \langle \mathbf{q}, K\mathbf{u}-\mathbf{g}\rangle_Y.
\end{equation}
The PDHG algorithm for solving \eqref{eq:minproblem} iteratively computes the solution of the associated saddle-point problem \eqref{eq:primaldualproblem2} by means of
\begin{equation*}
\begin{aligned}
  \mathbf{u}^{n+1} &= \prox_{\tau_n G}(\mathbf{u}^n - \tau_n K^{\top}\mathbf{q}^n ),\\
  \bar{\mathbf{u}}^{n+1} &=  \mathbf{u}^{n+1} + \left(\mathbf{u}^{n+1}-\mathbf{u}^n\right),\\
  \mathbf{g}^{n+1} &= \prox_{\frac{1}{\sigma_n}\|\cdot\|_{\vec{b},a}}\left(K\bar{\mathbf{u}}^{n+1} + \frac{\mathbf{q}^n}{\sigma_n}\right),\\
  \mathbf{q}^{n+1} &= \mathbf{q}^n + \sigma_n(K\bar{\mathbf{u}}^{n+1} - \mathbf{g}^{n+1}),
\end{aligned}
\end{equation*}
where $n\geq 0$ is the iteration number, and $\tau_n, \sigma_n>0$ are the step-size parameters. The algorithm basically consists of alternating a gradient descent in the primal variables $\mathbf{u}$ and $\mathbf{g}$, and a gradient ascent in the dual variable $\mathbf{q}$. 


\section{Applications to Image Processing} \label{sec:results}

We present an extensive performance evaluation of different CTV based methods on several inverse problems in color imaging such as denoising, deblurring, and inpainting. In these cases, one typically introduces a positive weighting constant $\lambda\geq 0$ that controls the trade-off between $G$, which forces the solution of the optimization problem to be close to some given data, and the regularization term:
$$
\min_{\mathbf{u}\in X} \frac{\lambda}{2}G(\mathbf{u}) +  \|K\mathbf{u}\|_{\vec{b},a}.
$$
For the sake of consistency among comparisons, we solved each problem with a range of different values of $\lambda$ and only reported the best result for each regularization and each degradation condition in terms of the highest peak signal-to-noise ratio (PSNR). Furthermore, we chose the linear operator $K$ to be the discrete local gradient computed via forward differences. In all tests, we used images from the Kodak collection (\url{http://r0k.us/graphics/kodak/}), and all results were saved in integer values relative to the intensity range $[0,255]$.

In view of the optimality conditions of \eqref{eq:primaldualproblem2}, one defines the following sequences of primal and dual residuals:
\begin{equation*}
\begin{aligned}
P_{n+1} &:= \dfrac{1}{\tau_n} \left(\mathbf{u}^n-\mathbf{u}^{n+1}\right) -K^{\top} \left( \mathbf{q}^{n}-\mathbf{q}^{n+1}\right),\\
D_{n+1} &:= \dfrac{1}{\sigma_n} \left(\mathbf{q}^n-\mathbf{q}^{n+1}\right) - K\left( \mathbf{u}^n-\mathbf{u}^{n+1}\right).
\end{aligned}
\end{equation*}
As stopping criterion we used a tolerance value of $10^{-5}$ for the average of the above residuals per pixel. In any case, we stopped the algorithm after $1000$ iterations even if the tolerance was not reached.

\subsection{Image Denoising: $\CTV-\ell^2$ Model}\label{sec:denoisingL2}

We propose to extend the widely mentioned ROF model to color images by using CTV regularization. The primal problem is therefore given by
\begin{equation}\label{eq:TVL2denoising}
 \min_{\mathbf{u}\in X} \dfrac{\lambda}{2}\|\mathbf{u}-\mathbf{f}\|^2_2 + \| K\mathbf{u}\|_{\vec{b},a}.
 \end{equation}
The $\ell^2$ norm is the most suitable choice for suppressing Gaussian noise, since the energy \eqref{eq:TVL2denoising} corresponds to the maximum a posteriori estimate. The proximity operator of the fidelity term $G(\mathbf{u}):= \frac{\lambda}{2}\|\mathbf{u}-\mathbf{f}\|^2_2$ is
$$
\widehat{\mathbf{u}} = \prox_{\tau G}(\mathbf{u}) \: \Leftrightarrow\: \widehat{\mathbf{u}} = \dfrac{\mathbf{u} + \tau\lambda \mathbf{f}}{1 + \tau\lambda}.
$$

To determine the general behaviour of several CTV regularizations with respect to changing the balancing parameter, Figure \ref{fig:psnrvslmb} shows the plots of the PSNR each method achieved for certain values of $\lambda$. For these tests, we artificially added zero-mean Gaussian noise of standard deviation $25$ to a noise-free color image. One observes that the peaks of the PSNR curves of the regularizations using $\ell^{\infty,1,1}$, $(S^1, \ell^1)$, $\ell^{2,1,1}$, and $\ell^{\infty,2,1}$ norms achieve the highest values. Interestingly, although $\ell^{1,1,1}$ shows one of the lowest performances in terms of the maximal PSNR, its corresponding curve seems to drop slower as one overestimates $\lambda$. As it is well known, the optimal value of $\lambda$ does not always lead to a complete noise removal. However, a huge reduction of the balancing parameter provides an over-smoothed result and, thus, significant information is lost. In the end, the optimal value in terms of the PSNR is obtained as a compromise between removing noise and preserving signal content.

\begin{figure}[!htbp]
\begin{center}
  \includegraphics[width=0.55\linewidth]{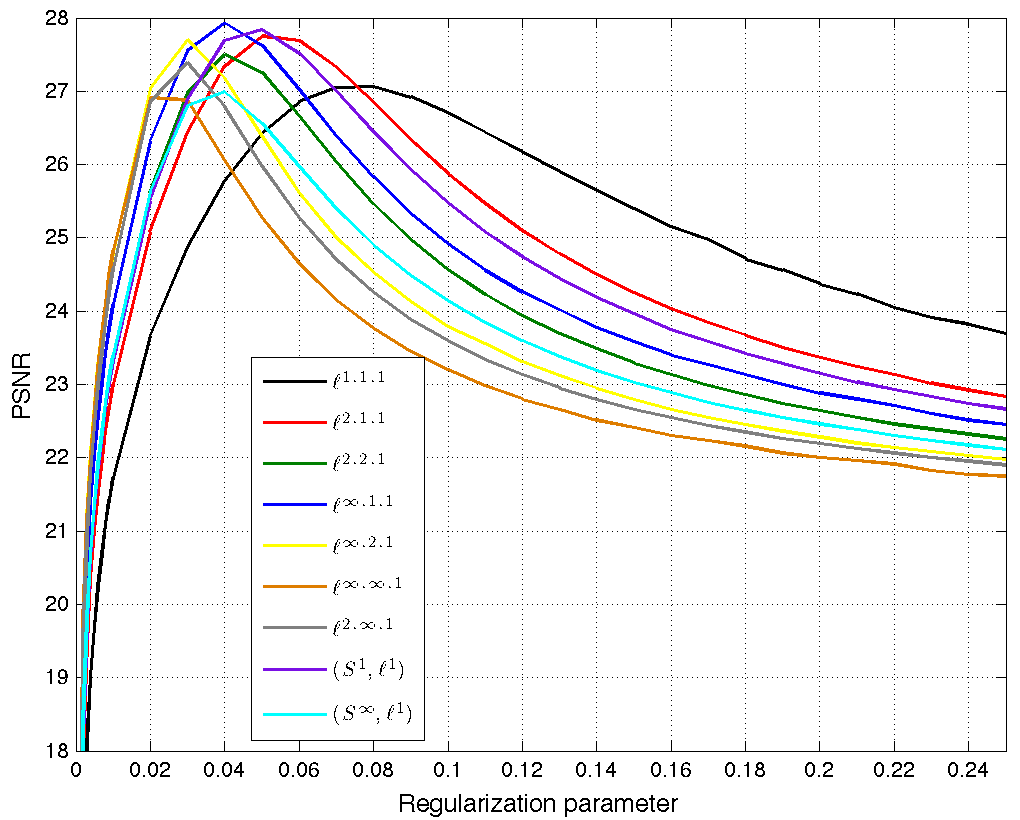}
\caption{Comparison of CTV methods using different values of $\lambda$. The peaks of the curves of $\ell^{\infty,1,1}$, $(S^1, \ell^1)$, $\ell^{2,1,1}$ and $\ell^{\infty,2,1}$ achieve the highest values. Although $\ell^{1,1,1}$ shows one of the lowest performances in terms of the maximal PSNR, its curve drops slower as one overestimates $\lambda$.}
\label{fig:psnrvslmb}
\end{center}
\end{figure}

As an example of our experiments on $\CTV-\ell^2$ denoising, we artificially added Gaussian noise with standard deviation $30$ to the twenty-third Kodak image and computed the PSNR value for each reconstruction by comparing to the noise-free image. Picking the optimal value of $\lambda$ in terms of the PSNR for each method, we obtained the results shown in Figure \ref{fig:denoisingL2}. We clearly observe that the CTV regularization based on the $\ell^{\infty,1,1}$ norm provides the best PSNR value, and its denoised image is superior to the others in visual quality. Indeed, see the strong color artifacts on the parrot's cheek for all results except for the $\ell^{\infty,1,1}$ norm. Although the $(S^1,\ell^1)$ norm shows nice denoising properties, a derivative matrix which has two derivative vectors being equal to zero also has rank one such that colored edges are not actively suppressed.
The large inter-channel correlation of images in the Kodak dataset explains why the $\ell^{\infty,1,1}$ norm, which encourages jumps that occur in all channels in the sense given in Section \ref{sec:singVectors}, performs visually the best. On the other hand, $\ell^{1,1,1}$ shows one of the worst performances since it neither couples the colors nor the derivatives. Furthermore, $\left( S^{\infty}, \ell^1\right)$ does not work very well. It seems that imposing jumps of different color channels to point into the same direction can more effectively be enforced by the convex relaxation $(S^1, \ell^1)$ than having a single direction in the dual variable as in the $\left(S^{\infty}, \ell^1\right)$ approach. Finally, the isotropic $\ell^{2,2,1}$ is beaten by the anisotropic $\ell^{2,1,1}$, and the new-proposed $\ell^{2,\infty, 1}(der, col, pix)$ outperforms $\ell^{\infty,2,1}(col, der, pix)$ in terms of both PSNR and visual quality assessment.

\begin{figure}[!htbp]
\begin{center}
\begin{tabular}{cc}
	\includegraphics[trim= 5.9cm 8.8cm 18.7cm 7cm, clip=true, scale=1.57]{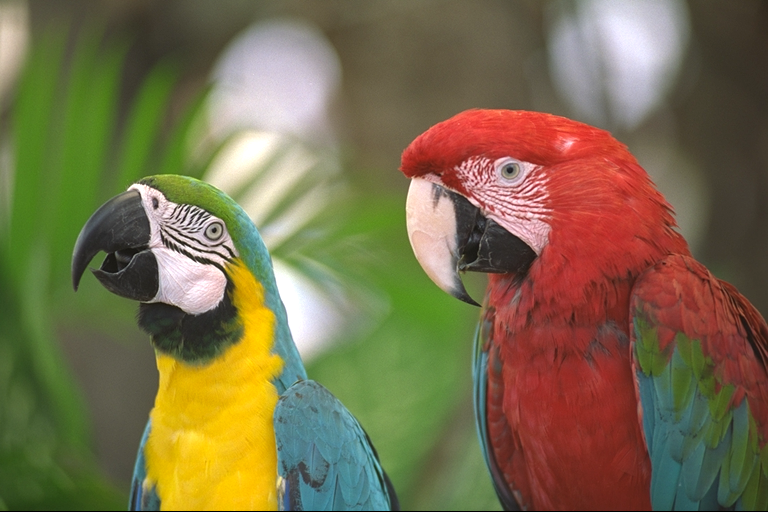} &
	\includegraphics[trim= 5.9cm 8.8cm 18.7cm 7cm, clip=true, scale=1.57]{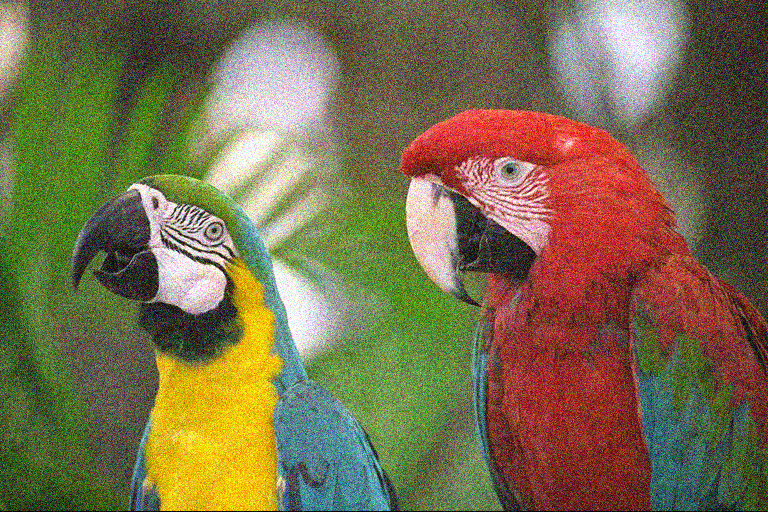}\\
	Clean & Noisy
\end{tabular}
\begin{tabular}{ccc}
	\includegraphics[trim= 5.9cm 8.8cm 18.7cm 7cm, clip=true, scale=1.57]{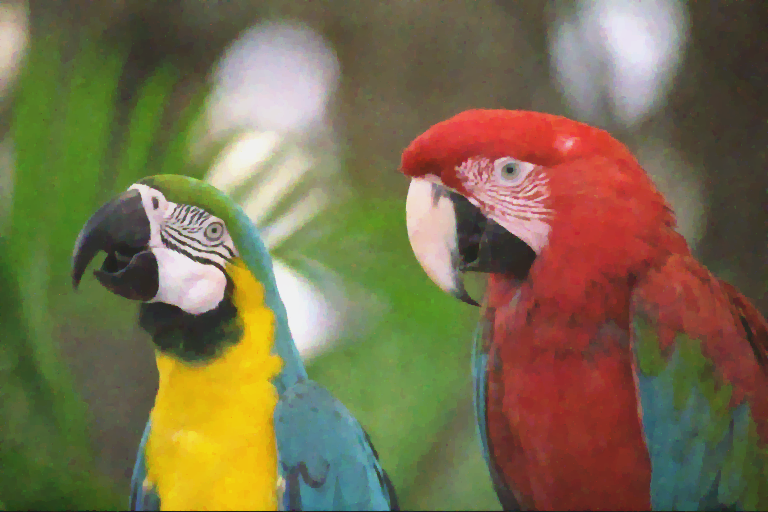} &
	\includegraphics[trim= 5.9cm 8.8cm 18.7cm 7cm, clip=true, scale=1.57]{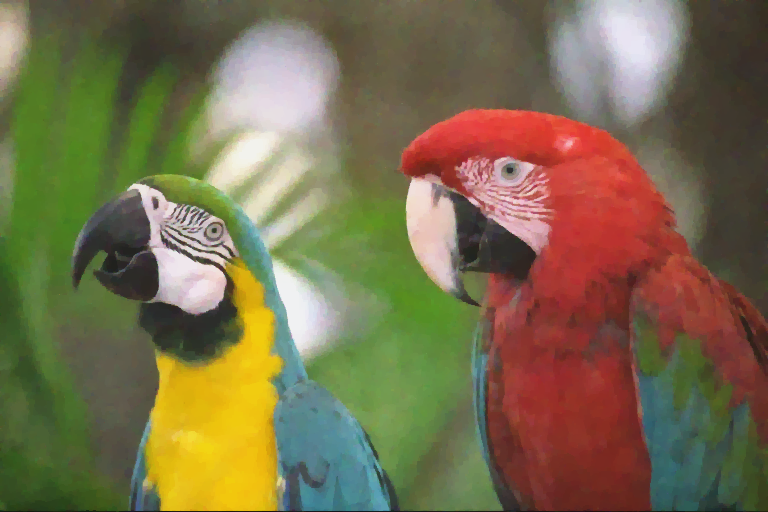} &
	\includegraphics[trim= 5.9cm 8.8cm 18.7cm 7cm, clip=true, scale=1.57]{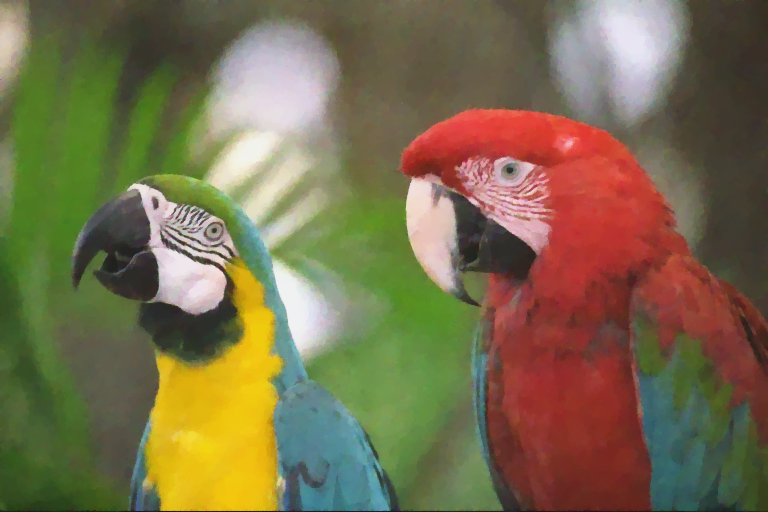} \\
	$\begin{array}{c} \ell^{1,1,1}(col,der,pix) \\ \text{PSNR} = 30.14 \end{array}$ & $\begin{array}{c}\ell^{2,1,1}(col,der,pix) \\ \text{PSNR} = 31.00 \end{array}$  & $\begin{array}{c} \ell^{2,2,1}(col,der,pix) \\ \text{PSNR} = 30.92\end{array}$ \\
	\includegraphics[trim= 5.9cm 8.8cm 18.7cm 7cm, clip=true, scale=1.57]{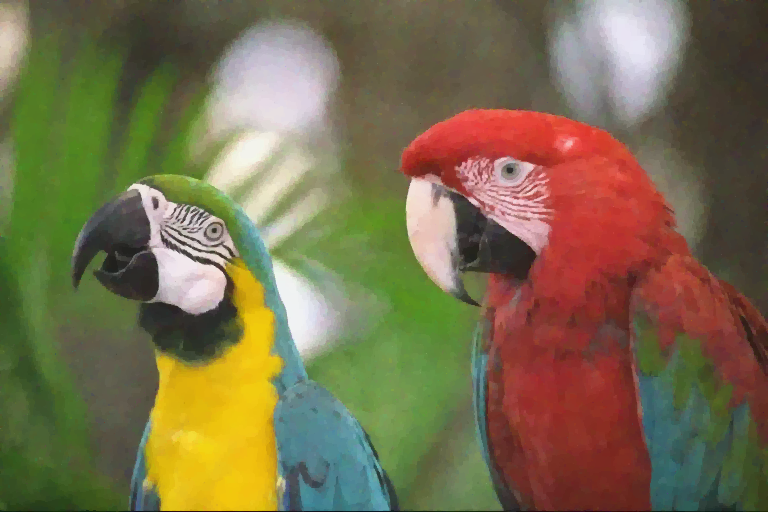} &
	\includegraphics[trim= 5.9cm 8.8cm 18.7cm 7cm, clip=true, scale=1.57]{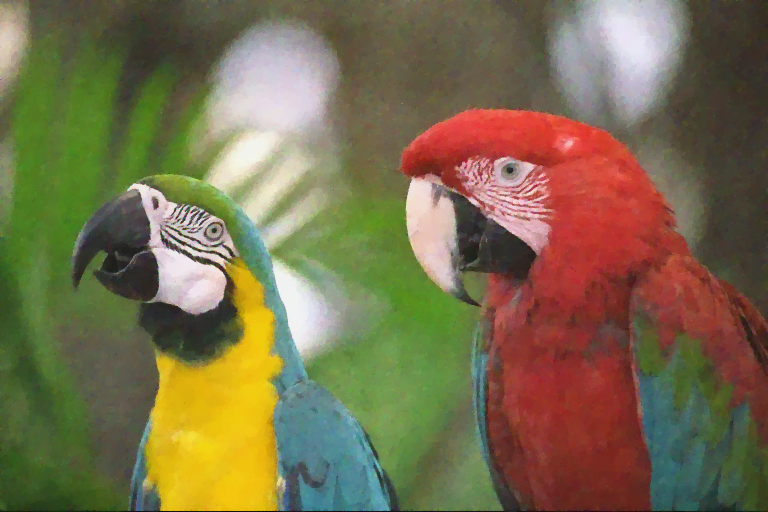} &
	\includegraphics[trim= 5.9cm 8.8cm 18.7cm 7cm, clip=true, scale=1.57]{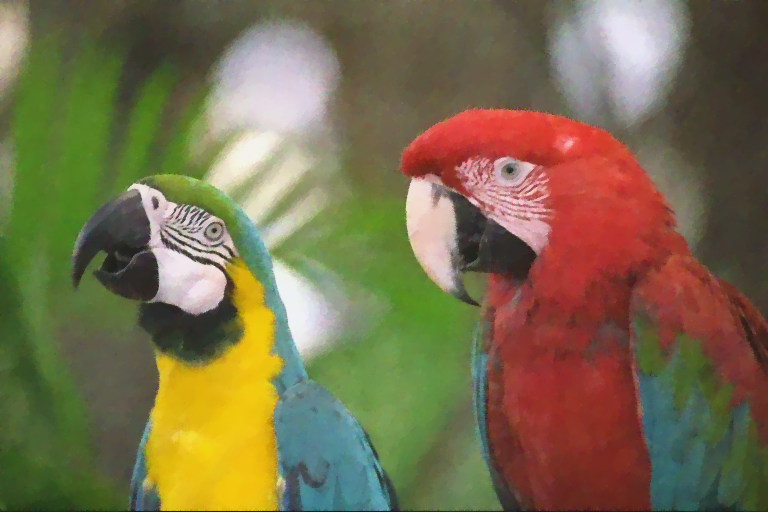} \\
	$\begin{array}{c} \ell^{\infty,1,1}(col,der,pix) \\ \text{PSNR} = 31.13 \end{array}$ & $\begin{array}{c}\ell^{\infty,2,1}(col,der,pix) \\ \text{PSNR} = 30.91 \end{array}$ & $\begin{array}{c} \ell^{\infty,\infty,1}(col,der,pix) \\ \text{PSNR} = 30.71\end{array}$ \\
	\includegraphics[trim= 5.9cm 8.8cm 18.7cm 7cm, clip=true, scale=1.57]{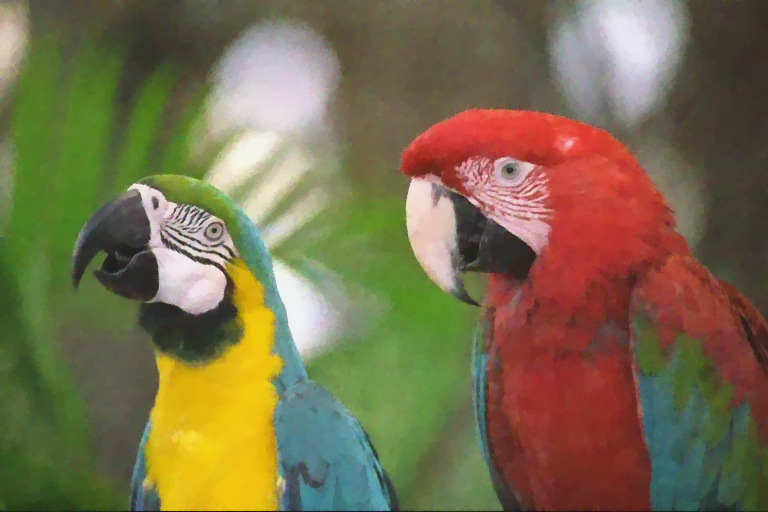}&
	\includegraphics[trim= 5.9cm 8.8cm 18.7cm 7cm, clip=true, scale=1.57]{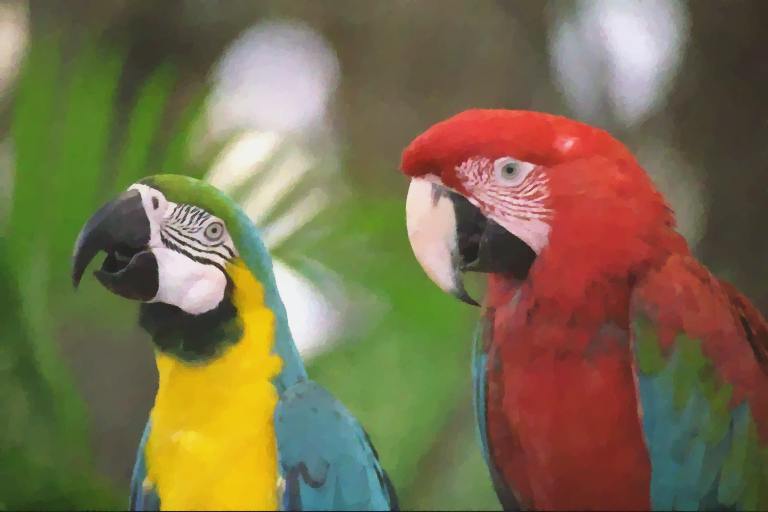} &
	\includegraphics[trim= 5.9cm 8.8cm 18.7cm 7cm, clip=true, scale=1.57]{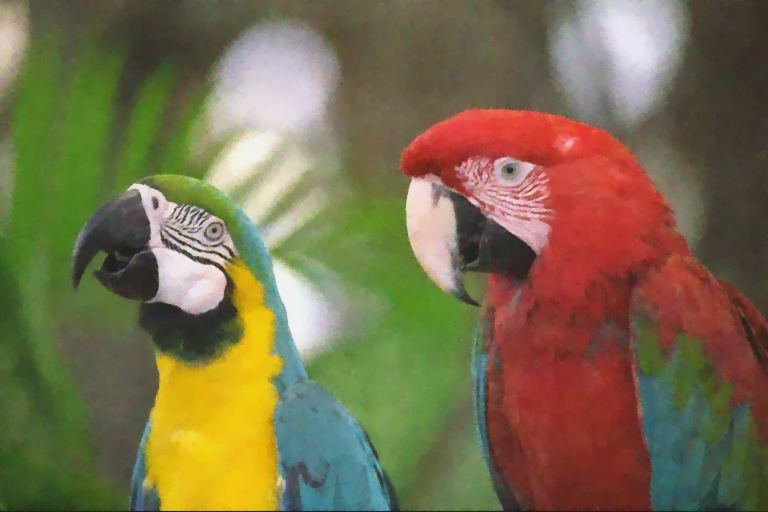} \\
	$\begin{array}{c} \ell^{2,\infty,1}(der,col,pix) \\ \text{PSNR} = 30.97 \end{array}$ & $\begin{array}{c} \left(S^1(col,der), \ell^1(pix)\right) \\ \text{PSNR} = 31.05 \end{array}$ & $\begin{array}{c} \left(S^{\infty}(col,der), \ell^1(pix)\right) \\ \text{PSNR} = 30.46 \end{array}$ 
\end{tabular}
\caption{Close-ups of the ground truth, the input noisy data (additive Gaussian noise of s.d. $30$), and the denoised images obtained from the minimization of \eqref{eq:TVL2denoising} on the twenty-third Kodak image. For each method, the value of $\lambda$ which gave the best PSNR value was determined experimentally. The PSNR value for each result is noted below the image. We observe that strong color artifacts remain on the parrot's cheek in all results except for the $\ell^{\infty,1,1}$ norm. In fact, this method gives rise to significantly better visual quality and provides the best PSNR value. Although $(S^1,\ell^1)$ most closely approaches $\ell^{\infty,1,1}$ in terms of the numerical assessment, it is still far from suppressing spots and color artifacts as $\ell^{\infty,1,1}$ does.}
\label{fig:denoisingL2}
\end{center}
\end{figure}

A more detailed comparison analysis on color image denoising by the CTV$-\ell^2$ model, supporting software and an online demo will be made available soon. 

\subsection{Image Denoising: $\CTV-\ell^1$ Model}

If we replace the $\ell^2$ norm in the data-penalty term of \eqref{eq:TVL2denoising} by the more robust $\ell^1$ norm, the $\CTV-\ell^1$ model arises:
\begin{equation}\label{eq:TVL1denoising}
 \min_{\mathbf{u}\in X} \lambda\|\mathbf{u}-\mathbf{f}\|_1 +  \| K\mathbf{u}\|_{\vec{b},a}.
 \end{equation}
Some well-known advantages of \eqref{eq:TVL1denoising} over the classical ROF model are contrast invariance and more effectiveness in removing noise containing strong outliers such as the salt-and-pepper type noise. In this case, the proximity operator of the fidelity term $G(\mathbf{u}):= \lambda\|\mathbf{u}-\mathbf{f}\|_1$ is
$$
\widehat{\mathbf{u}} = \prox_{\tau G}(\mathbf{u}) \: \Leftrightarrow\: \widehat{u}_{i,k} = \left\lbrace \begin{array}{ll} u_{i,k} - \tau\lambda & \text{if }\, u_{i,k}-f_{i,k} > \tau\lambda, \\ u_{i,k}+\tau\lambda & \text{if }\, u_{i,k}-f_{i,k} < -\tau\lambda, \\ f_{i,k} & \text{if }\, |u_{i,k}-f_{i,k}|\leq \tau\lambda. \end{array}\right.
$$
Note that the $\CTV-\ell^1$ model poses a nonsmooth optimization problem, which is also treatable by the PDHG algorithm.

Given the probability $\alpha\in[0,1]$ that a pixel is corrupted, we introduced salt-and-pepper noise by setting a fraction of $\frac{\alpha}{2}$ randomly selected pixels to black, and another fraction of $\frac{\alpha}{2}$ randomly selected pixels to white. We display in Figure \ref{fig:denoisingL1} the optimal result each method provided on parts of the fifth Kodak image for $\alpha=0.15$. At first glance, the regularization using the newly-proposed $\ell^{\infty,1,1}$ norm is the most successful in suppressing color spots. The numerical results confirm the previous visual inspection, since the PSNR value associated to the denoised image given by $\ell^{\infty,1,1}$ is clearly superior to all others. In fact, this is the unique method that actively suppresses the input noise and preserves sharp edges. For instance, observe that the edges separating saturated regions, such as the contours of the green and yellow front mudguards, are specially damaged with all regularizations except $\ell^{\infty,1,1}$. Finally, it is worth stressing that $\ell^{2,\infty,1}$ clearly outperforms $\ell^{\infty,2,1}$.

\begin{figure}[!htbp]
\begin{center}
\begin{tabular}{cc}
	\includegraphics[trim= 5.5cm 5cm 15.5cm 7.5cm, clip=true, scale=0.64]{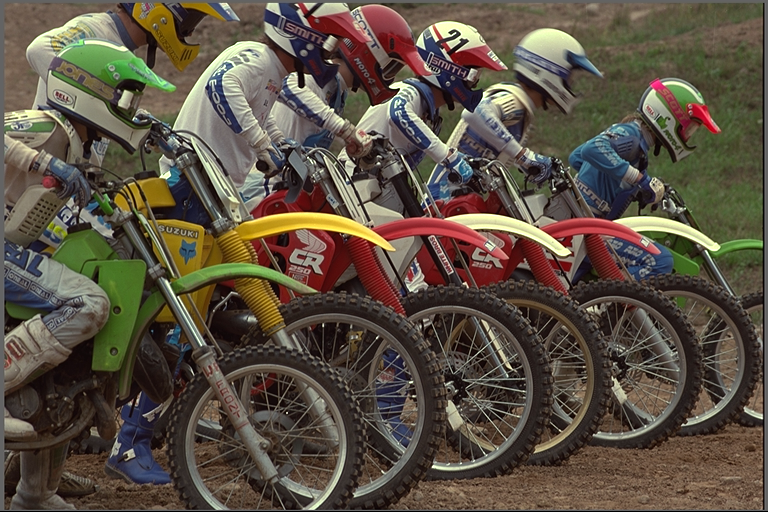} &
	\includegraphics[trim= 5.5cm 5cm 15.5cm 7.5cm, clip=true, scale=0.64]{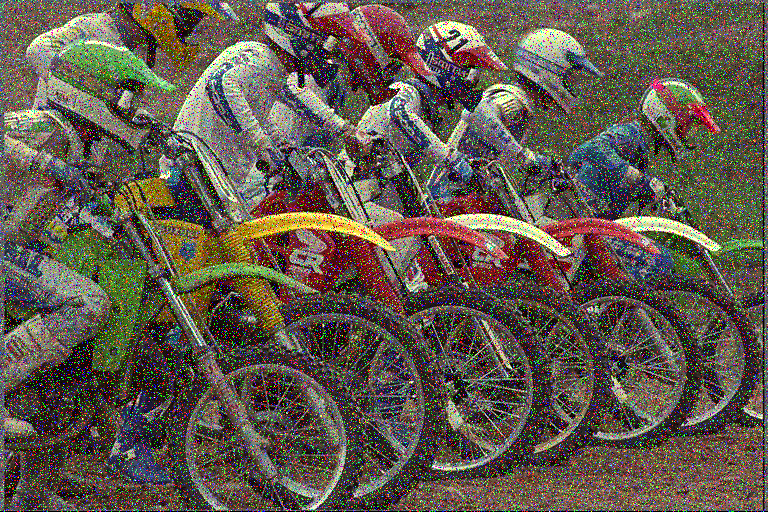}\\
	Clean & Noisy
\end{tabular}
\begin{tabular}{ccc}
	\includegraphics[trim= 5.5cm 5cm 15.5cm 7.5cm, clip=true, scale=0.64]{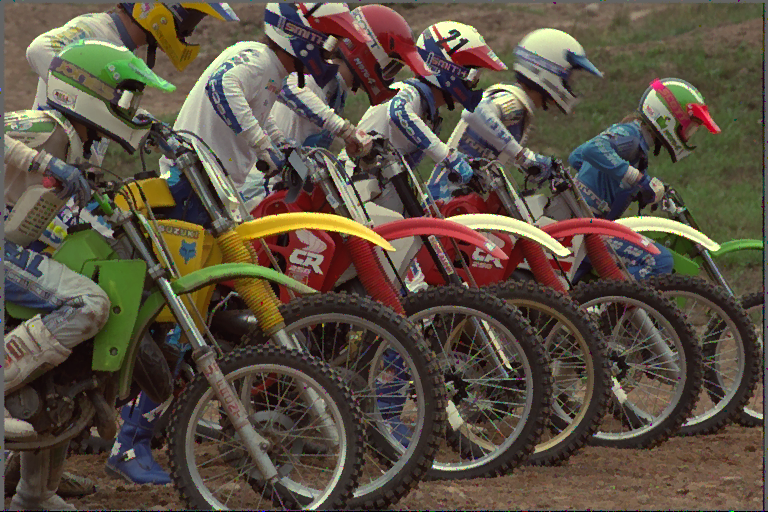} &
	\includegraphics[trim= 5.5cm 5cm 15.5cm 7.5cm, clip=true, scale=0.64]{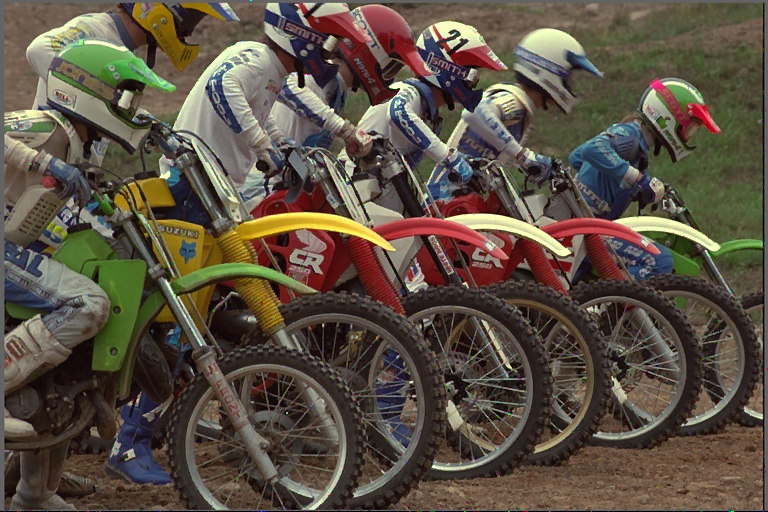} &
	\includegraphics[trim= 5.5cm 5cm 15.5cm 7.5cm, clip=true, scale=0.64]{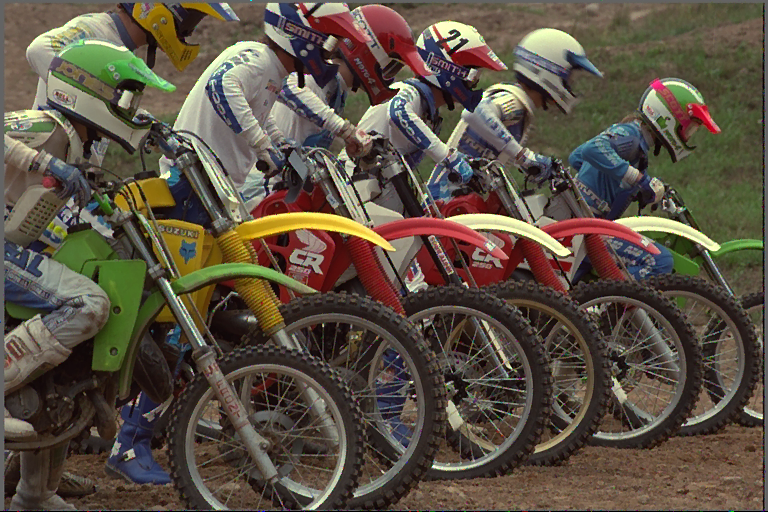} \\
	$\begin{array}{c} \ell^{1,1,1}(col,der,pix) \\ \text{PSNR} = 26.40 \end{array}$ & $\begin{array}{c}\ell^{2,1,1}(col,der,pix) \\ \text{PSNR} = 29.33 \end{array}$  & $\begin{array}{c} \ell^{2,2,1}(col,der,pix) \\ \text{PSNR} = 28.77\end{array}$ \\
	\includegraphics[trim= 5.5cm 5cm 15.5cm 7.5cm, clip=true, scale=0.64]{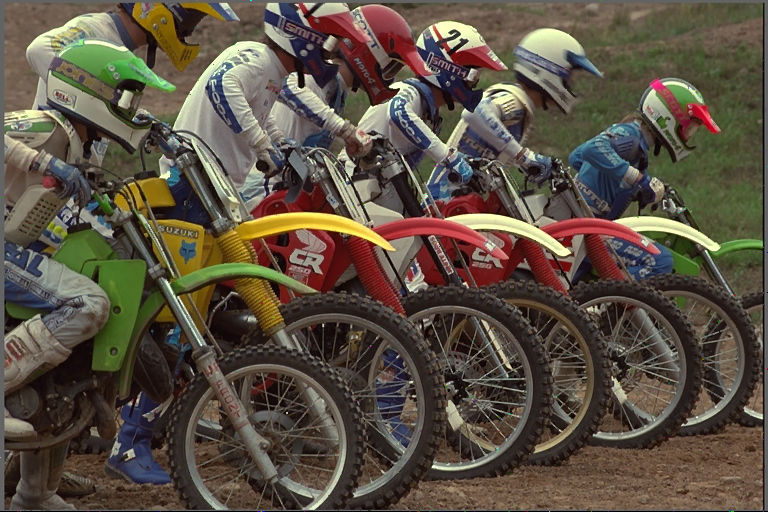} &
	\includegraphics[trim= 5.5cm 5cm 15.5cm 7.5cm, clip=true, scale=0.64]{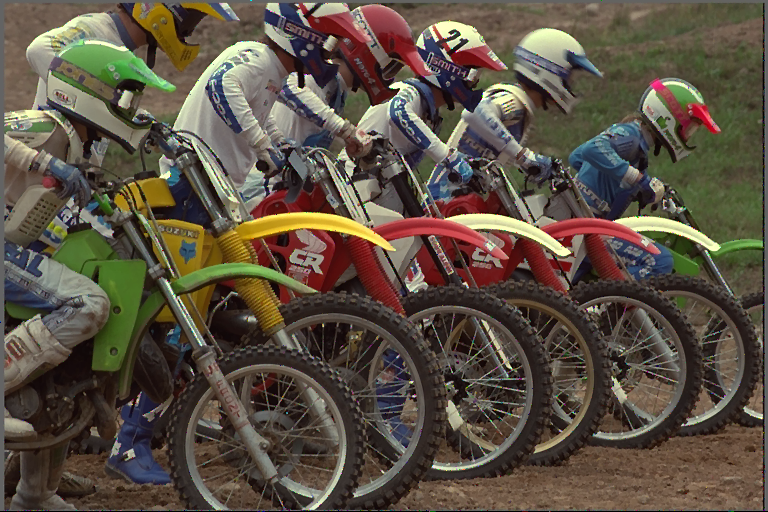} &
	\includegraphics[trim= 5.5cm 5cm 15.5cm 7.5cm, clip=true, scale=0.64]{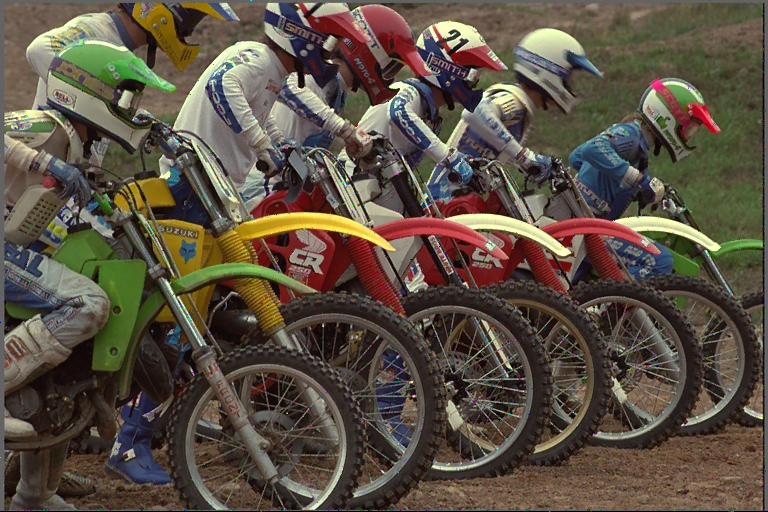} \\
	$\begin{array}{c} \ell^{\infty,1,1}(col,der,pix) \\ \text{PSNR} = 31.67 \end{array}$ & $\begin{array}{c}\ell^{\infty,2,1}(col,der,pix) \\ \text{PSNR} = 28.62 \end{array}$ & $\begin{array}{c} \ell^{\infty,\infty,1}(col,der,pix) \\ \text{PSNR} = 29.75\end{array}$ \\
	\includegraphics[trim= 5.5cm 5cm 15.5cm 7.5cm, clip=true, scale=0.64]{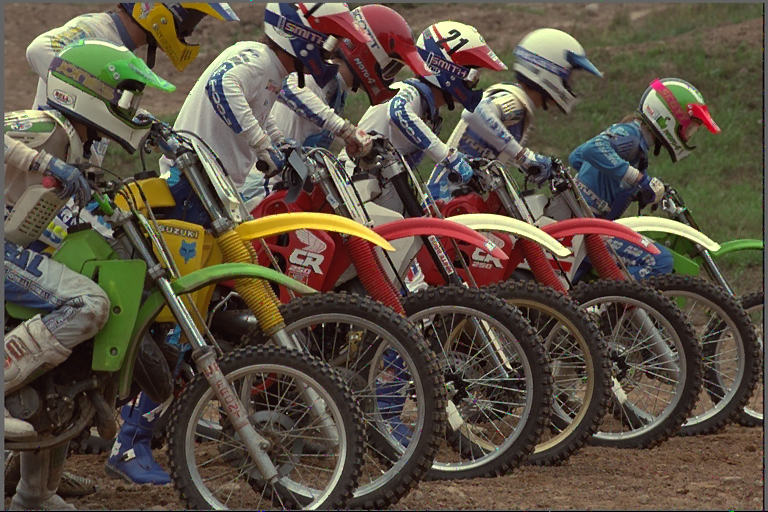}&
	\includegraphics[trim= 5.5cm 5cm 15.5cm 7.5cm, clip=true, scale=0.64]{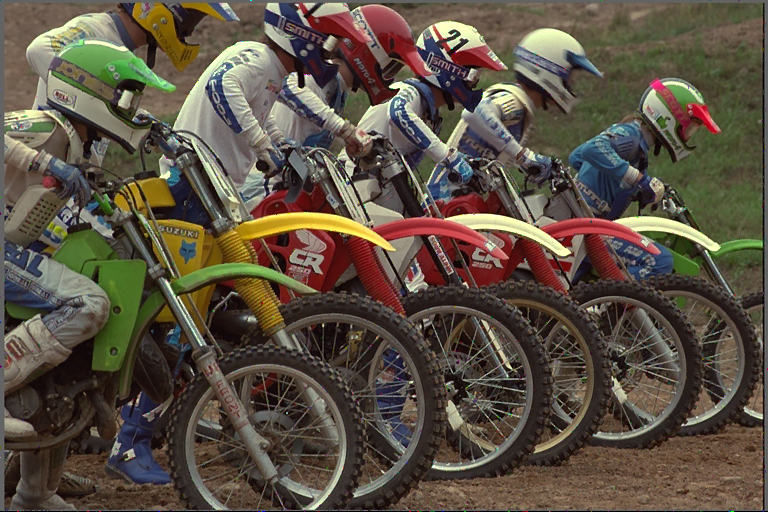} &
	\includegraphics[trim= 5.5cm 5cm 15.5cm 7.5cm, clip=true, scale=0.64]{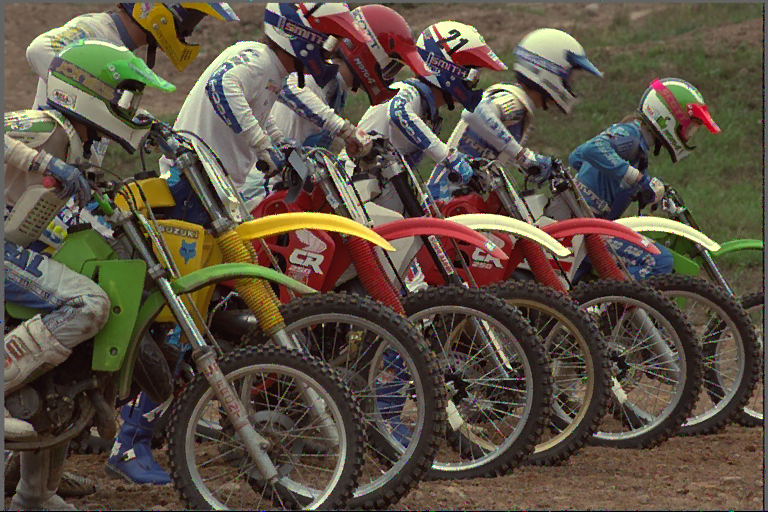} \\
	$\begin{array}{c} \ell^{2,\infty,1}(der,col,pix) \\ \text{PSNR} = 30.50 \end{array}$ & $\begin{array}{c} \left(S^1(col,der), \ell^1(pix)\right) \\ \text{PSNR} = 30.86 \end{array}$ & $\begin{array}{c} \left(S^{\infty}(col,der), \ell^1(pix)\right) \\ \text{PSNR} = 27.15 \end{array}$ 
\end{tabular}
\caption{Close-ups of the ground truth, the input noisy data (15\% of pixels with salt-and-pepper noise), and the denoised images obtained from the minimization of \eqref{eq:TVL1denoising} on the fifth Kodak image. For each method, the value of $\lambda$ which gave the best PSNR value was determined experimentally. The PSNR value for each result is noted below the image. We observe that the $\ell^{\infty,1,1}$ norm is clearly superior in terms of the error as well from a visual inspection. Indeed, it is the most effective regularization to remove spots and preserve colors. Note also that $\ell^{\infty,1,1}$ and (to a lesser extend) $\ell^{2,\infty,1}$ produce denoised images with sharp contours, which does not happen in all other cases since edges separating colored regions are damaged.}
\label{fig:denoisingL1}
\end{center}
\end{figure}

\subsection{Image Deblurring}
The extension of the variational ROF model for image deblurring involves the minimization of the primal energy
$$
\min_{\mathbf{u}\in X}  \dfrac{\lambda}{2} \| A\mathbf{u} -\mathbf{f} \|^2_2 + \|K\mathbf{u}\|_{\vec{b},a},
$$
where $A$ is a linear operator modeling the degradation of $\mathbf{u}$ caused by blur and possibly noise. For the following experiments, we focus on image deconvolution, which refers to the case where the blur to be removed is linear and shift-invariant so that it may be expressed as a convolution of the image with a point spread function. Accordingly, the linear operator is given by $A\mathbf{u} = \varphi \ast \mathbf{u}$, where $\varphi$ is a Gaussian convolution kernel.

The proximal mapping of the fidelity term $G(\mathbf{u}):=\frac{\lambda}{2}\|A\mathbf{u}-\mathbf{f}\|^2_2$ is given by
\begin{equation}\label{eq:proxblur}
\widehat{\mathbf{u}}= \text{prox}_{\tau G}(\mathbf{u})  \: \Leftrightarrow\: (I + \tau\lambda A^{*}A)\widehat{\mathbf{u}} = \mathbf{u} + \tau\lambda A^{*}\mathbf{f}.
\end{equation}
Note that the above formula requires to compute $(I + \tau\lambda A^{*}A)^{-1}$, which is huge time consuming in the spatial domain for large values of the standard deviation of the kernel. This drawback is solved by working in the Fourier domain where the convolution becomes a mere multiplication. Hence, using the convolution theorem of Fourier transforms, the solution of \eqref{eq:proxblur} can be efficiently computed as
\begin{equation}\label{eq:prox_deblurFFT}
\widehat{\mathbf{u}} = \mathcal{F}^{-1} \left(  \dfrac{\mathcal{F}(\mathbf{u}) + \tau\lambda \mathcal{F}(A)^{*} F(\mathbf{f})}{1 + \tau\lambda \mathcal{F}(A)^2}\right),
\end{equation}
where $\mathcal{F}$ and $\mathcal{F}^{-1}$ denote the Fast Fourier Transform (FFT) and the inverse FFT, respectively. Note that all operations in the above formula are componentwise.

We tested all CTV regularizations on the third Kodak image. The degraded data was simulated by convolving the ground truth with a Gaussian kernel of standard deviation $2$ and further adding white Gaussian noise of standard deviation $0.5$. The quality of the restored images with optimal values of $\lambda$ can be evaluated both visually and numerically in Figure \ref{fig:deblurring}. We observe that the blur has been almost suppressed  in all cases even though some geometry and texture cannot be recovered from the corrupted data. As expected from any TV based model, the restored images tend to be piecewise smooth. In general terms, it seems that isotropic regularization is more suitable for image deblurring -- at least with very little noise -- than anisotropic filtering. Indeed, $\ell^{2,2,1}$ and the new-proposed $\ell^{\infty,2,1}$ provide the best PSNR values together with the nuclear norm $(S^1,\ell^1)$. On the other hand, one realizes that $\ell^{\infty,1,1}$ and $(S^1, \ell^1)$ are superior in removing color artifacts at the text on the cap. In the end, the nuclear norm compromises between removing blur and avoiding color spots. 

\begin{figure}[!htbp]
\begin{center}
\begin{tabular}{cc}
	\includegraphics[trim= 5cm 7cm 16.2cm 5.5cm, clip=true, scale=0.65]{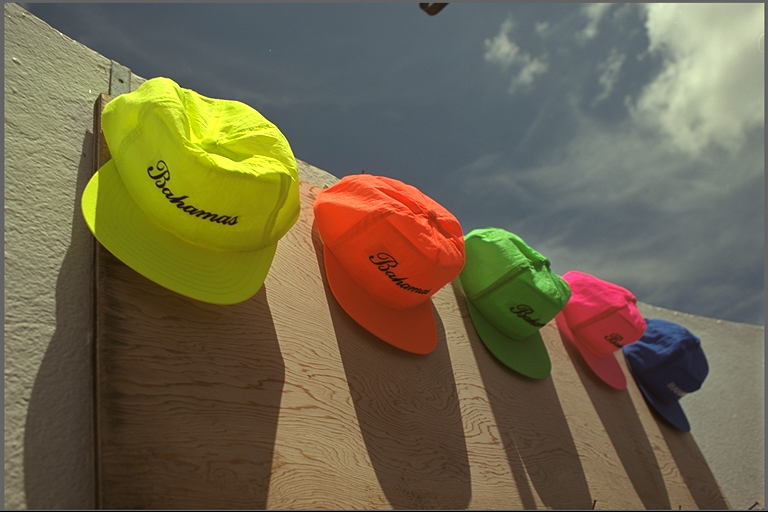} &
	\includegraphics[trim= 5cm 7cm 16.2cm 5.5cm, clip=true, scale=0.65]{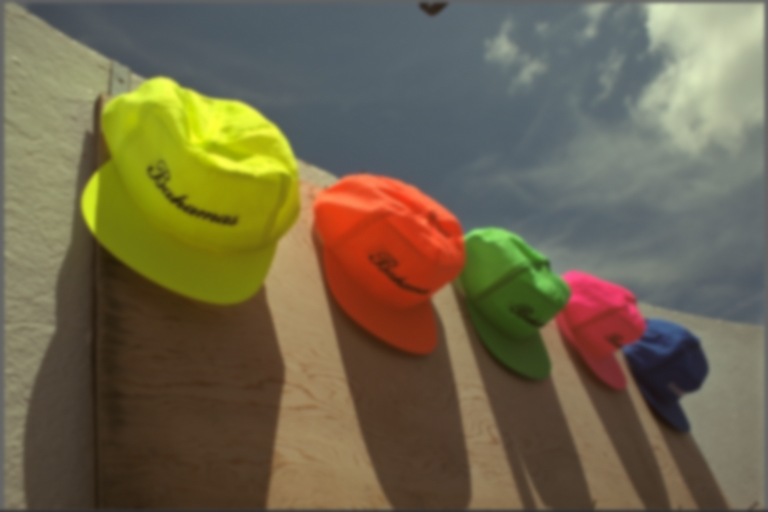}\\
	Clean & Blurred and noisy
\end{tabular}
\begin{tabular}{ccc}
	\includegraphics[trim= 5cm 7cm 16.2cm 5.5cm, clip=true, scale=0.65]{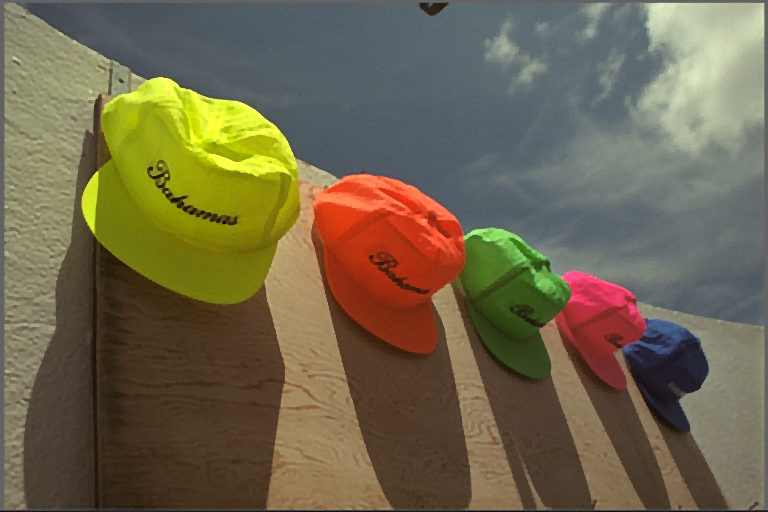} & 
	\includegraphics[trim= 5cm 7cm 16.2cm 5.5cm, clip=true, scale=0.65]{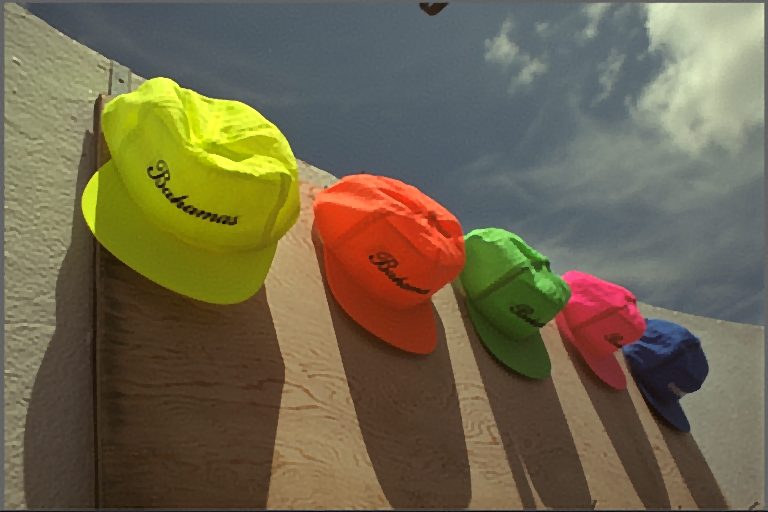} &
	\includegraphics[trim= 5cm 7cm 16.2cm 5.5cm, clip=true, scale=0.65]{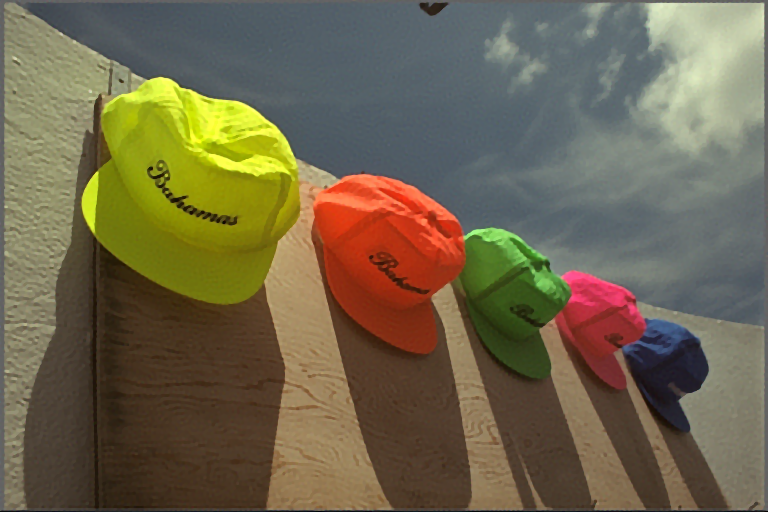} \\
	$\begin{array}{c} \ell^{1,1,1}(col,der,pix) \\ \text{PSNR} = 32.16 \end{array}$ & $\begin{array}{c}\ell^{2,1,1}(col,der,pix) \\ \text{PSNR} = 32.56 \end{array}$  & $\begin{array}{c} \ell^{2,2,1}(col,der,pix) \\ \text{PSNR} = 32.76\end{array}$ \\
	\includegraphics[trim= 4.75cm 7cm 16.2cm 5.5cm, clip=true, scale=0.65]{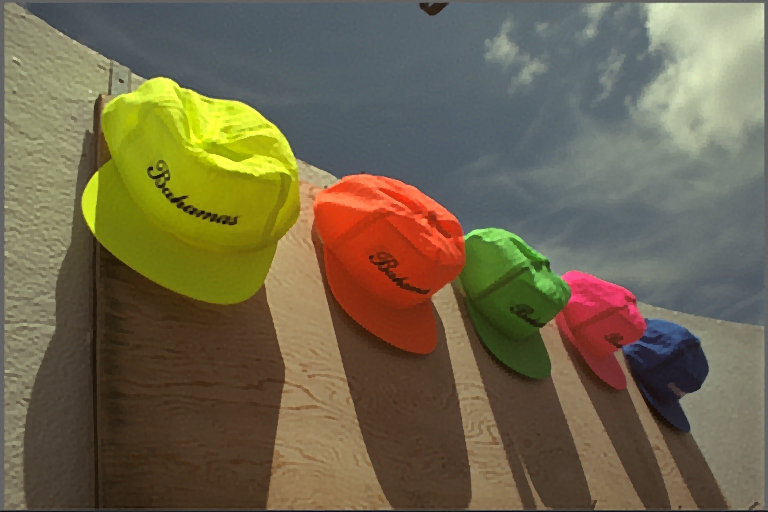} &
	\includegraphics[trim= 5cm 7cm 16.2cm 5.5cm, clip=true, scale=0.65]{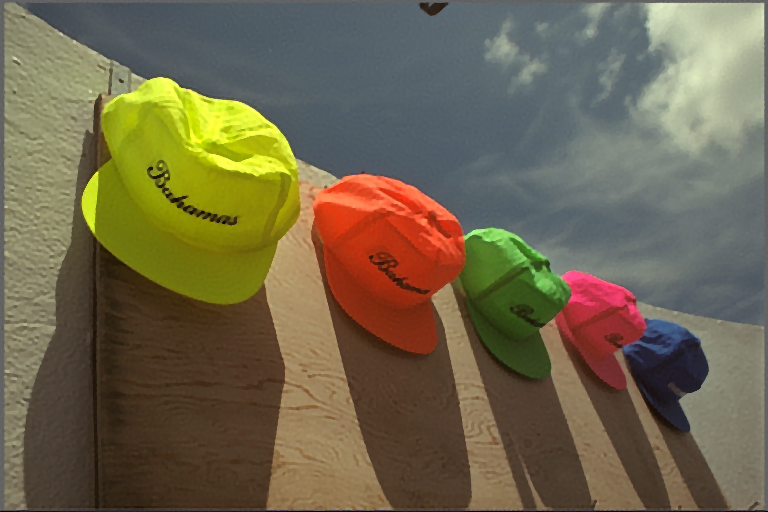} &
	\includegraphics[trim= 5cm 7cm 16.2cm 5.5cm, clip=true, scale=0.65]{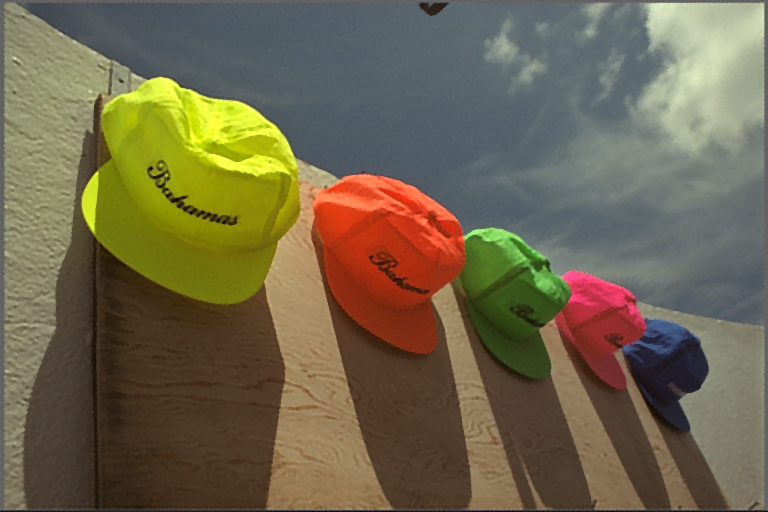} \\
	$\begin{array}{c} \ell^{\infty,1,1}(col,der,pix) \\ \text{PSNR} = 32.59 \end{array}$ & $\begin{array}{c}\ell^{\infty,2,1}(col,der,pix) \\ \text{PSNR} = 32.71 \end{array}$ & $\begin{array}{c} \ell^{\infty,\infty,1}(col,der,pix) \\ \text{PSNR} = 32.45\end{array}$ \\
	\includegraphics[trim= 5cm 7cm 16.2cm 5.5cm, clip=true, scale=0.65]{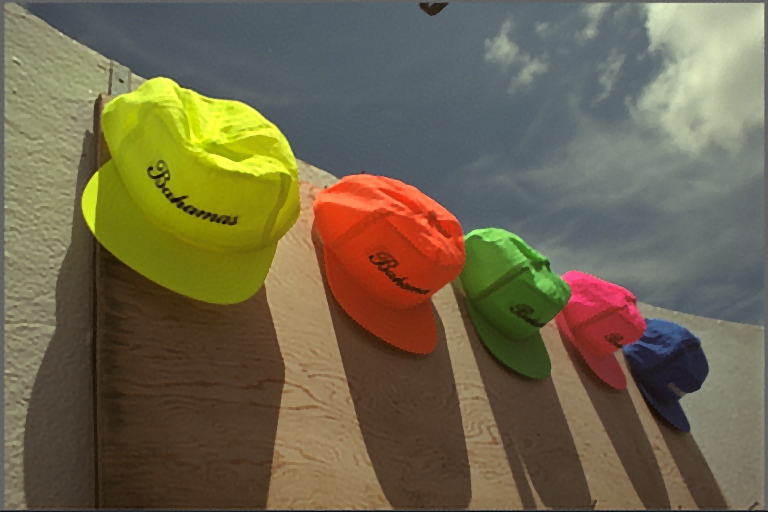}&
	\includegraphics[trim= 5cm 7cm 16.2cm 5.5cm, clip=true, scale=0.65]{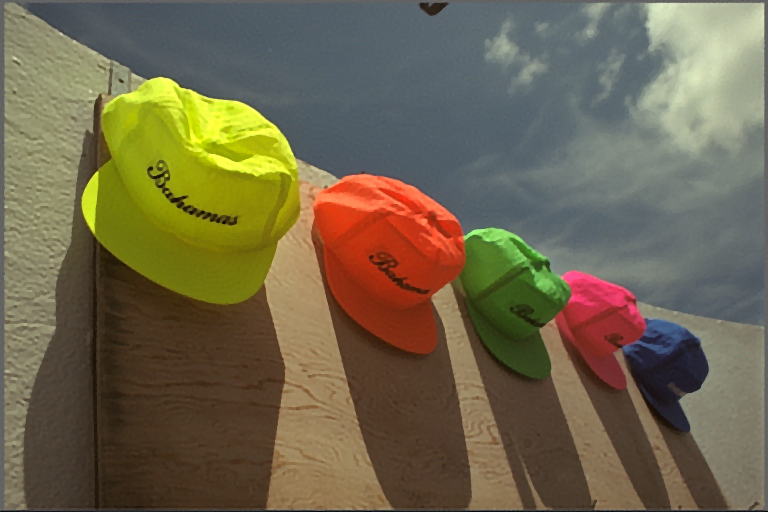} &
	\includegraphics[trim= 5cm 7cm 16.2cm 5.5cm, clip=true, scale=0.65]{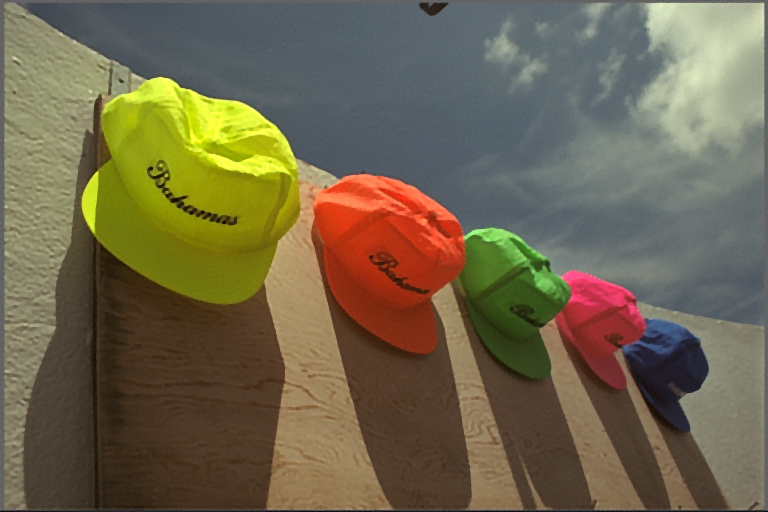} \\
	$\begin{array}{c} \ell^{2,\infty,1}(der,col,pix) \\ \text{PSNR} = 32.67 \end{array}$ & $\begin{array}{c} \left(S^1(col,der), \ell^1(pix)\right) \\ \text{PSNR} = 32.77 \end{array}$ & $\begin{array}{c} \left(S^{\infty}(col,der), \ell^1(pix)\right) \\ \text{PSNR} = 32.57 \end{array}$ 
\end{tabular}
\caption{Close-ups of the ground truth, the input blurred and noisy data (Gaussian convolution of s.d. $2$ and further additive Gaussian noise of s.d. $0.5$), and the restored images each method provided on the third Kodak image. For each CTV regularization, the value of $\lambda$ which gave the best PSNR value was determined experimentally. The PSNR value for each result is noted below the image. We observe that the blur has been almost suppressed  in all cases although some spatial details cannot be recovered from the corrupted data. As expected from TV based models, the restored images tend to be piecewise smooth. Note also that the visual and numerical differences are not as great as in the denoising case. However, it seems that isotropic diffusion is more suitable for deblurring images with very little noise since $\ell^{2,2,1}$ and the new-proposed $\ell^{\infty,2,1}$ provide the best results together with the nuclear norm $(S^1,\ell^1)$.}
\label{fig:deblurring}
\end{center}
\end{figure}

\subsection{Image Inpainting}
Image inpainting is the process of filling-in lost data in a known region of an image. Although during the last years a lot of effort has been put into the development of powerful image priors, we are interested in the TV based image inpainting model \cite{ChanShen2001}, which is limited to inpainting the geometric structure at unknown pixels.

Let $\mathcal{I}\subset\R^{N}$ be the inpainting domain, that is, the set of all pixels in the image where the intensity value of all color channels is unknown. Therefore, the primal problem we focus on is given by
\begin{equation}\label{eq:TVinpainting}
\min_{\mathbf{u}\in X}  \dfrac{\lambda}{2}\|\mathbf{u}-\mathbf{f}\|^2_{\R^N\setminus\mathcal{I}} + \|K\mathbf{u} \|_{\vec{b},a},
\end{equation}
where $\|\cdot\|_{\R^N\setminus\mathcal{I}}$ denotes the Euclidean norm at known pixels. We see that the proximity operator of $G(\mathbf{u})=\frac{\lambda}{2}\|\mathbf{u}-\mathbf{f}\|^2_{\R^N\setminus\mathcal{I}}$ is 
$$
\widehat{\mathbf{u}} = \prox_{\tau G}(\mathbf{u}) \: \Leftrightarrow\: \widehat{u}_{i,k} = \left\lbrace \begin{array}{ll} u_{i,k} & \text{if } i\in\mathcal{I},\vspace{0.2cm} \\ \dfrac{u_{i,k}+\tau\lambda f_{i,k}}{1+\tau\lambda} & \text{otherwise}.\end{array}\right.
$$

For the comparative quality assessment in image inpainting, we used a mask with random scribbles. In Figure \ref{fig:inpainting} we show the optimal result in terms of the highest PSNR provided by each CTV regularization on parts of the twentieth Kodak image. Since the image domain which is to be filled in is thin, pretty good numerical results are in general obtained. Indeed, all methods exhibit excellent PSNR since an increase of about $20$ dB is reached (the value of the input data is $20.58$). Concerning $\ell^{p,q,r}$ norms, one realizes that isotropic regularization performs significantly better than anisotropic filtering. In this setting, observe that $\ell^{1,1,1}$, $\ell^{2,1,1}$, and $\ell^{\infty,1,1}$ provide the lowest PSNR values as well the worst inpainted images from visual quality assessment. On the other hand,  CTV methods based on $\ell^{2,2,1}$, $(S^{\infty}, \ell^1)$, $(S^1, \ell^1)$, and $\ell^{2,\infty,1}$ norms are significantly superior to all other regularizations both visually -- compare the results at the edge separating the two gray regions with different color scheme -- and in terms of the metric. Accordingly, TV-based inpainting prefers straight contours as they have minimal total variation, but it is less successful for recovering curved boundaries. In this setting, one sees that all methods perfectly recover the color edge in the propeller of the plane, but they fail at its yellow boundary.

\begin{figure}[!htbp]
\begin{center}
\begin{tabular}{cc}
	\includegraphics[trim= 7.5cm 6.5cm 15cm 7cm, clip=true, scale=0.84]{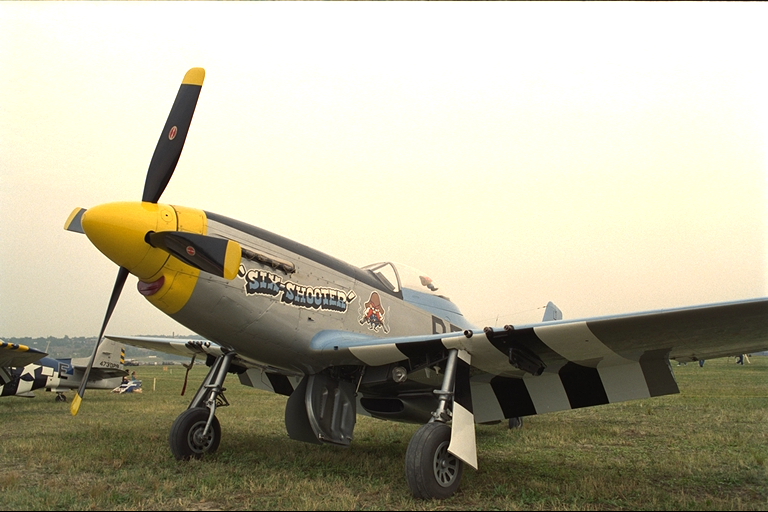} &
	\includegraphics[trim= 7.5cm 6.5cm 15cm 7cm, clip=true, scale=0.84]{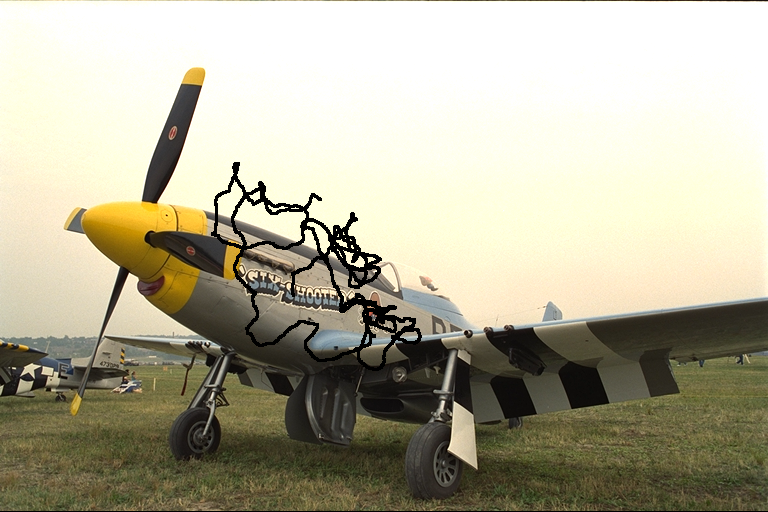}\\
	Clean & Masked
\end{tabular}
\begin{tabular}{ccc}
	\includegraphics[trim= 7.5cm 6.5cm 15cm 7cm, clip=true, scale=0.84]{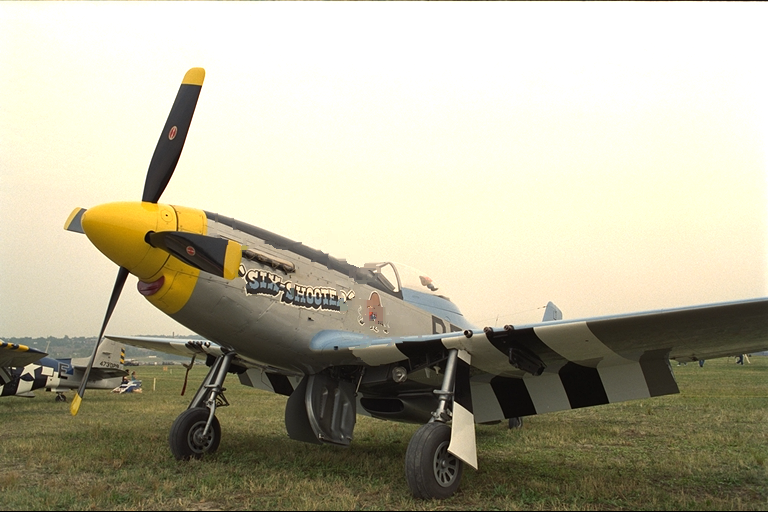} & 
	\includegraphics[trim= 7.5cm 6.5cm 15cm 7cm, clip=true, scale=0.84]{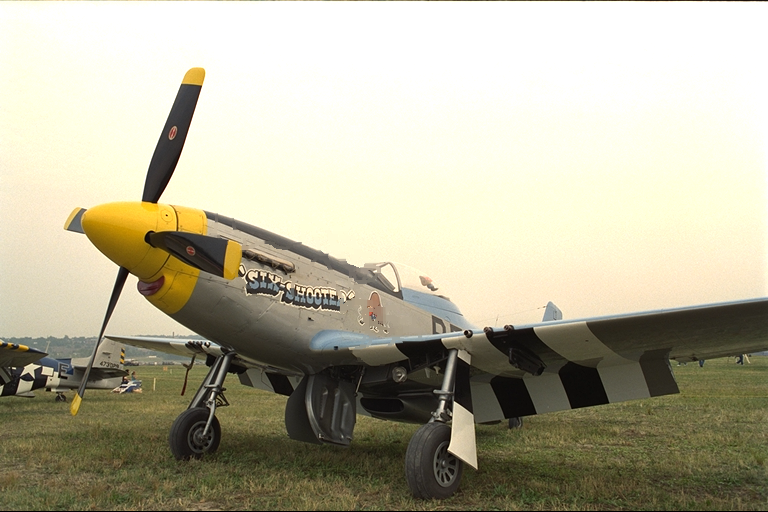} &
	\includegraphics[trim= 7.5cm 6.5cm 15cm 7cm, clip=true, scale=0.84]{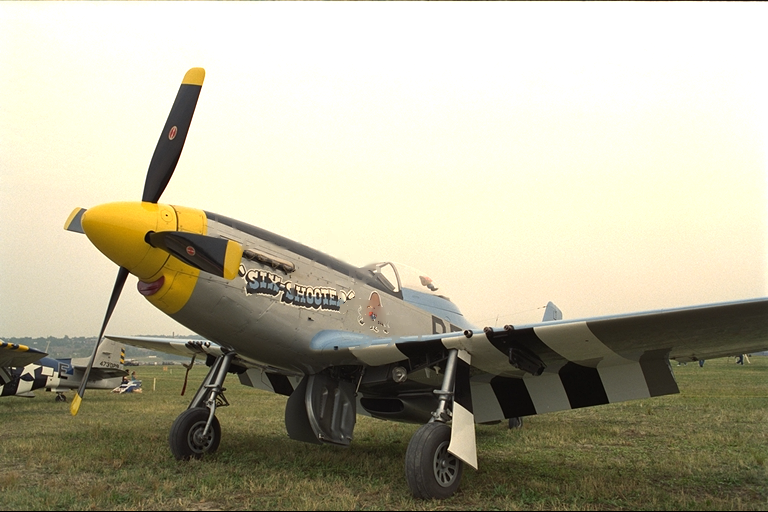} \\
	$\begin{array}{c} \ell^{1,1,1}(col,der,pix) \\ \text{PSNR} = 37.97 \end{array}$ & $\begin{array}{c}\ell^{2,1,1}(col,der,pix) \\ \text{PSNR} = 38.15 \end{array}$  & $\begin{array}{c} \ell^{2,2,1}(col,der,pix) \\ \text{PSNR} = 39.15\end{array}$ \\
	\includegraphics[trim= 7.5cm 6.5cm 15cm 7cm, clip=true, scale=0.84]{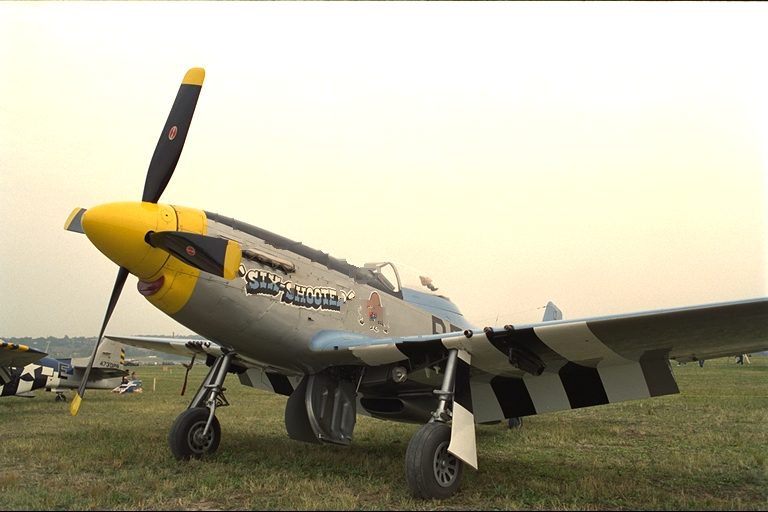} &
	\includegraphics[trim= 7.5cm 6.5cm 15cm 7cm, clip=true, scale=0.84]{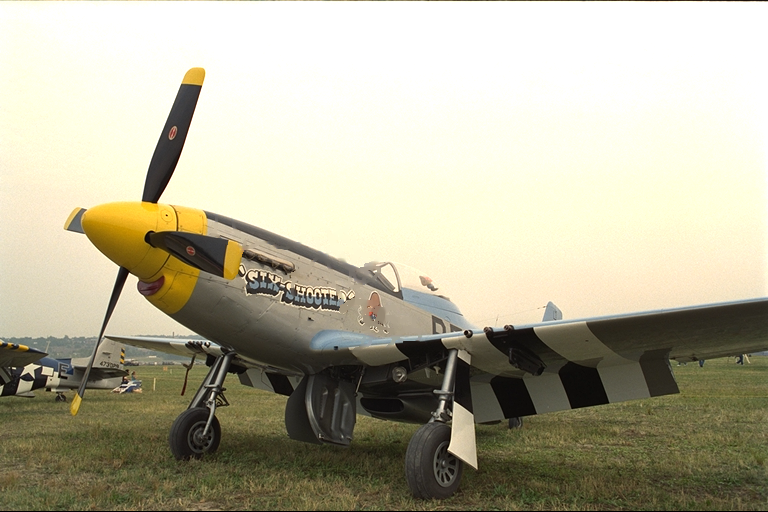} &
	\includegraphics[trim= 7.5cm 6.5cm 15cm 7cm, clip=true, scale=0.84]{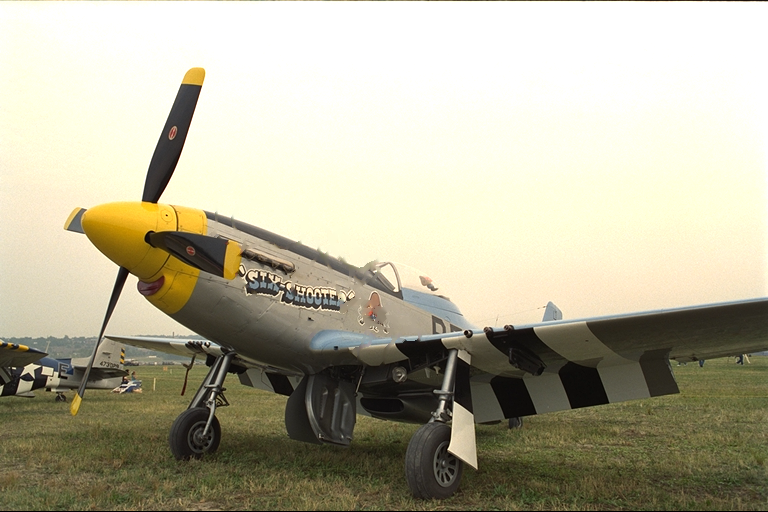} \\
	$\begin{array}{c} \ell^{\infty,1,1}(col,der,pix) \\ \text{PSNR} = 37.84 \end{array}$ & $\begin{array}{c}\ell^{\infty,2,1}(col,der,pix) \\ \text{PSNR} = 38.23 \end{array}$ & $\begin{array}{c} \ell^{\infty,\infty,1}(col,der,pix) \\ \text{PSNR} = 37.88\end{array}$ \\
	\includegraphics[trim= 7.5cm 6.5cm 15cm 7cm, clip=true, scale=0.84]{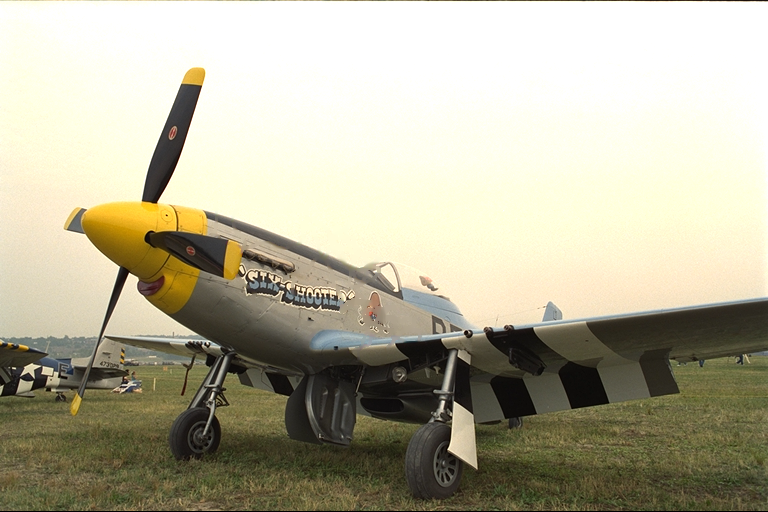}&
	\includegraphics[trim= 7.5cm 6.5cm 15cm 7cm, clip=true, scale=0.84]{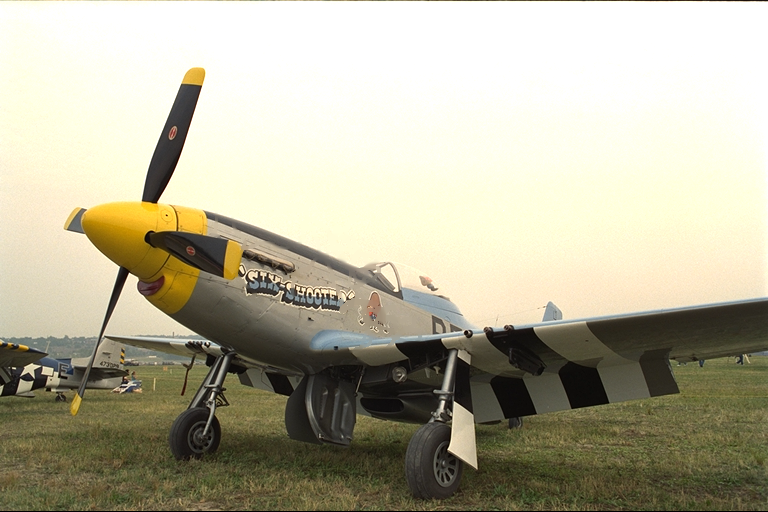} &
	\includegraphics[trim= 7.5cm 6.5cm 15cm 7cm, clip=true, scale=0.84]{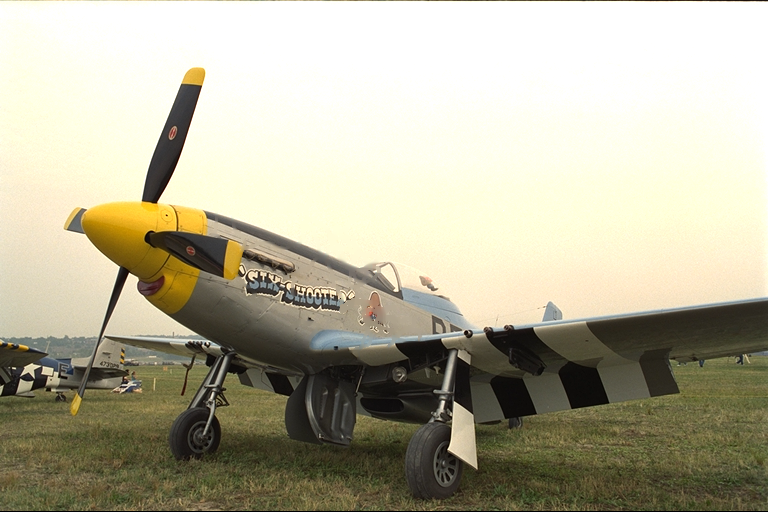} \\
	$\begin{array}{c} \ell^{2,\infty,1}(der,col,pix) \\ \text{PSNR} = 38.95 \end{array}$ & $\begin{array}{c} \left(S^1(col,der), \ell^1(pix)\right) \\ \text{PSNR} = 39.09 \end{array}$ & $\begin{array}{c} \left(S^{\infty}(col,der), \ell^1(pix)\right) \\ \text{PSNR} = 39.10 \end{array}$ 
\end{tabular}
\caption{Close-ups of the ground truth, the input masked image with random scribble, and the inpainted images provided by the minimization of \eqref{eq:TVL2denoising} on the twentieth Kodak image. For each method, the value of $\lambda$ which gave the best PSNR value was determined experimentally. The PSNR value for each result is noted below the image. In this case, we realize that anisotropic diffusion severely compromises the perfomance of $\ell^{p,q,r}$ norms. Indeed, $\ell^{1,1,1}$, $\ell^{2,1,1}$ and $\ell^{\infty,1,1}$ provide the lowest PSNR values as well the worst inpainted images from visual quality assessment. On the contrary, $\ell^{2,2,1}$, $(S^{\infty},\ell^1)$, $(S^1,\ell^1)$ and $\ell^{2,\infty,1}$ give similar results and are able to better recover sharp edges damaged by the mask.} 
\label{fig:inpainting}
\end{center}
\end{figure}

{\bf CIE-$\text{L}^{*}\text{a}^{*}\text{b}^{*}$ space.} All previous experiments were performed using the standard RGB color space. The CIE-$\text{L}^{*}\text{a}^{*}\text{b}^{*}$ is a perceptually uniform color space, a property which the common RGB model does not have, describing all the colors visible to the human eye. The three coordinates $L^{*}$, $a^{*}$ and $b^{*}$ represent the lightness of the color, its position between red/magenta and green, and its position between yellow and blue, respectively. Contrary to RGB color systems, in this space the color differences which one perceives correspond to distances when measuring colorimetrically.

Figure \ref{fig:cielab} shows the results of minimizing \eqref{eq:TVinpainting}  on both RGB and CIE-$\text{L}^{*}\text{a}^{*}\text{b}^{*}$ color spaces by means of $\CTV-\ell^{2,\infty,1}$ regularization, which benefits from the superiority of isotropic diffusion as demonstrated in Figure \ref{fig:inpainting}. We observe that the choice of the uniform color space leads to a slight but visually noticeable improvement. Indeed, the bleeding of red across edges in RGB space vanishes when transforming the image into CIE-$\text{L}^{*}\text{a}^{*}\text{b}^{*}$ before inpainting. Accordingly, the PSNR gain is not negligible. 

\begin{figure}[!htbp]
\begin{center}
\begin{tabular}{ccc}
	\includegraphics[trim= 7.5cm 8.9cm 18.6cm 8.3cm, clip=true, scale=3.9]{inpainting/clean.png} &
	\includegraphics[trim= 7.5cm 8.9cm 18.6cm 8.3cm, clip=true, scale=3.9]{inpainting/l2inf1.png} &
	\includegraphics[trim= 7.5cm 8.9cm 18.6cm 8.3cm, clip=true, scale=3.9]{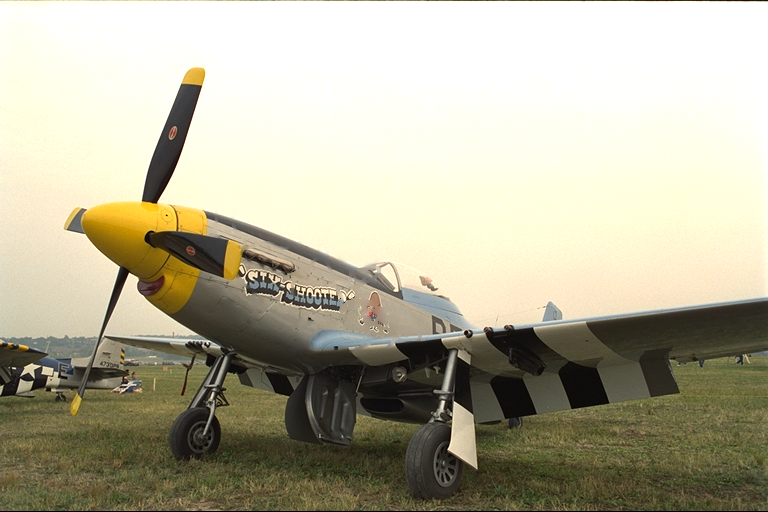} \\
	Clean & $\begin{array}{c} \ell^{2,\infty,1} \text{ in RGB space}\\ \text{PSNR} = 38.95\end{array}$ & $\begin{array}{c} \ell^{2,\infty,1} \text{ in CIELab space}\\ \text{PSNR} = 39.18\end{array}$ 
\end{tabular}
\caption{Close-ups of the ground truth and the inpainted images provided by the minimization of \eqref{eq:TVinpainting} on the twentieth Kodak image. In each case, the value of $\lambda$ which gave the best PSNR value was determined experimentally. First, we note that the PSNR gain in the CIE-$\text{L}^{*}\text{a}^{*}\text{b}^{*}$ space is not at all negligible. But, more importantly, the visual quality assessment demonstrates that a bleeding of red across edges appears in the RGB space. This effect vanishes when the inpainting is carried out in the perceptually uniform CIE-$\text{L}^{*}\text{a}^{*}\text{b}^{*}$ color space.} 
\label{fig:cielab}
\end{center}
\end{figure}

\section{Conclusions}\label{sec:conclusions}
Considering the discrete setting, we have proposed to view the gradient of a multispectral image as a three dimensional matrix or tensor with the dimensions corresponding to the spatial extend, the directional derivatives considered as linear operators containing the differences to other pixels, and the color channels. We have then introduced collaborative total variation as the regularization that arises from taking different norms along each dimension. In particular, we have proposed to use collaborative norms such as $\ell^{p,q,r}$ and $(S^p,\ell^q)$, leading to very different properties of the regularization. We have provided relevant mathematical characterizations of the dual norm, the subdifferential and the proximal mapping of the proposed penalizations, which play a direct role in computing optimality conditions of several regularized problems. We have further proved, using the generalized concept of singular vectors, than an $\ell^{\infty}$ coupling leads to the strongest channel correlation, makes the most prior assumptions, and has the greatest potential to reduce color artifacts.

In experiments, we have demonstrated the wide applicability of the collaborative total variation to general inverse problems like denoising, deblurring and inpainting. For the numerical computation of the solution we have used the primal-dual hybrid gradient algorithm  and stated all proximity operators of the considered CTV regularizations. From the above standards, we have exhibited the superiority of the $\ell^{\infty,1,1}$ norm for a stronger suppression of color artifacts, and of the isotropic regularizations for filling in thin regions. 

\bibliographystyle{siam}
\bibliography{references}

\appendix

\vspace{1cm}
\section{Singular Vector Analysis}\label{app:singVectors}

We give here the mathematical details regarding the construction of singular vectors as discussed in Section \ref{sec:singVectors}. Let us remark that the following analysis could be done in a continuous setting with weak derivatives and distributions as long as there is a finite number of points at which the linearity of the functions $l$ changes. Since our whole discussion about collaborative norms has been dealing with the discrete case, we limit the proofs in this section to the discrete setup, too. 

Let $l: \{x_1, ..., x_N\} \rightarrow [-1,1]$ be a discretization of a piecewise linear function such that the piecewise linearity only changes at $\{-1,1\}$. More precisely, define the slope as well as the discrete derivative operator at each point as $s_i := D_x l(x_i) = l(x_{i})-l(x_{i-1})$. Let $D_x^T$ denote the adjoint operator of $D_x$ defined by analogy with the continuous setting: $\langle D^T_xl, f\rangle = \langle l, Df\rangle$. One checks easily that $D^T_xl(x_i)= l(x_i)-l(x_{i+1})$. Then, we require either $s_{i+1}=s_i$, which means that we are in the piecewise linear part, or $|l(x_i)|= 1$, which means that we are at a point where the type of linearity changes. As a first step, let us state the following lemma which will be needed in all following proofs.

\begin{lemma}\label{lem:signOfLinearFunction}
Using definitions and notations above, $s_i-s_{i+1}>0$ implies $l(x_i)=1$, and $s_i-s_{i+1}<0$ implies $l(x_i)=-1$. In particular, $l(x_i) D_xD^T_xl(x_i) = |s_i-s_{i+1}|$ and $l(x_i)=\sign(s_i-s_{i+1})$ whenever $s_i\neq s_{i+1}$. 
\end{lemma}

\begin{proof}
If $s_i-s_{i+1}>0$, then $|l(x_i)|=1$ due to definition of $l$. Let us suppose that $l(x_i)=-1$. From $|l(x_j)|\leq 1$ for all $j\in\{1,\ldots, N\}$, it follows that $l(x_{i-1})\geq -1$ and, as a consequence, $s_i  \leq 0$. This means that $s_{i+1} = l(x_{i+1})- l(x_i)<0$, that is, $l(x_{i+1})<-1$, which contradicts $|l(x_{i+1})|\leq 1$. We thus deduce that $l(x_i)=1$. The proof for the case $s_i-s_{i+1}<0$ can be done in a similar fashion. 

The additional statement is a simple consequence of the first part along with $D_xD^T_x l(x_i) = s_i-s_{i+1}$ by definition of the operators. 
\end{proof}

For each CTV regularization, the following results show that if $z_k^1$ and $z_k^2$ have some specific expressions, then the associated image $\mathbf{u}=D^T\mathbf{z}$ is a singular vector of the energy $J(\mathbf{u})=\|D\mathbf{u}\|_{\vec{b},a}$.

\begin{theorem} \label{thm:singular111}
For each $k\in\{1,\ldots,C\}$, let $l^1_k$ and $l^2_k$ be discretizations of arbitrary piecewise linear functions in $[-1,1]$, with the properties described previously. Let us consider $z^1_k(x_i,y_j) = c^1_k l^1_k(x_i)$ and $z^2_k(x_i,y_j) = c^2_k l^2_k(y_j)$ such that $c_k^1,c_k^2\in\{0,\pm 1\}$, as well as $u_k(x_i,y_j) = (D_x^T z^1_k)(x_i,y_j) + (D_y^T z^2_k)(x_i,y_j)$. Then, $\mathbf{u} \in \partial J(\mathbf{u})$ for $J(\mathbf{u}) = \|D\mathbf{u}\|_{1,1,1}$.
\end{theorem}

\begin{proof}
We aim at proving $\mathbf{u}\in\partial J(\mathbf{u})$ with $J(\mathbf{u})=\|D\mathbf{u}\|_{1,1,1}$. Based on the characterization of the subdifferential of $J$ given in \eqref{eq:subdiff1}, we have to show that $\|\mathbf{z}\|_{\infty,\infty,\infty}\leq 1$, which is obvious due to its construction, together with $\sign(D_x u_k(x_i,y_j)) = z^1_k(x_i,y_j)$ and $\sign(D_y u_k(x_i,y_j)) = z^2_k(x_i,y_j)$ for all $(x_i, y_j)$ at which $Du_k(x_i,y_j) \neq 0$. First, we can assume that $c_k^1\neq 0$ since, otherwise, the statement is trivially satisfied. Now, observe that
\begin{equation*}
\begin{aligned}
D_xu_k(x_i,y_j) &= D_xD^T_x z^1_k(x_i,y_j) + D_xD^T_y z^2_k(x_i,y_j) = c_k^1D_xD^T_x l^1_k(x_i) + c_k^2 D_xD^T_y l^2_k(y_j)\\
&= c_k^1 D_xD^T_x l^1_k(x_i) = c_k^1(s_i-s_{i+1}),
\end{aligned}
\end{equation*}
where $D_xD^T_y l^2_k(y_j) = 0$ since $D^T_y l^2_k(y_j)$ does not depend on $x$. Hence, $Du_k(x_i,y_j)\neq 0$ implies that $s_i\neq s_{i+1}$ and, thus, $l^1_k(x_i)=\sign(s_i-s_{i+1})$ by Lemma \ref{lem:signOfLinearFunction}. It follows that
$$
\sign(D_x u_k(x_i,y_j)) =  \sign\left(c_k^1(s_i-s_{i+1})\right) = c_k^1 l_k^1(x_i)=z_k^1(x_i,y_j),
$$
where in the second transition from last we have used $c_k^1=\sign(c_k^1)$ derived from $c_k^1\in\{-1,+1\}$. The proof of $\sign(D_y u_k(x_i,y_j)) = z_k^2(x_i,y_j)$ is similar and yields the assertion $\langle \mathbf{z}, D\mathbf{u}\rangle=\|D\mathbf{u}\|_{1,1,1}$.
\end{proof}

\begin{theorem} \label{thm:singular211}
Let $l^1$ and $l^2$ be discretizations of arbitrary piecewise linear functions in $[-1,1]$, with the properties described previously. For each $k\in\{1,\ldots, C\}$, let us define $z_k^1(x_i,y_j) =  c^1_k l^1(x_i)$ and $z_k^2(x_i,y_j) = c^2_k l^2(y_j)$ such that $c^1_k, c^2_k \in \mathbb{R}$ with $\|c^1\|_2=\|c^2\|_2=1$, and $u_k(x_i,y_j) =(D_x^T z_k^1)(x_i,y_j) + (D_y^T z_k^2)(x_i,y_j)$. Then, $\mathbf{u}\in\partial J(\mathbf{u})$ for $J(\mathbf{u}) = \|D\mathbf{u}\|_{2,1,1}$.
\end{theorem}

\begin{proof}
Similar to Theorem \ref{thm:singular111}, we have to check that $\|\mathbf{z}\|_{2,\infty,\infty}\leq 1$, which follows easily from $\|c^r\|_2=1$ and $|l^r(x)|\leq 1$, as well as $\langle \mathbf{z}, D\mathbf{u}\rangle = \|D\mathbf{u}\|_{2,1,1}$. Using Lemma \ref{lem:signOfLinearFunction}, we see that
\begin{equation*}
\begin{aligned}
\sqrt{\sum_{k=1}^C(D_x u_k(x_i,y_j))^2} &= \sqrt{\sum_{k=1}^C(D_x D_x^T z_k(x_i,y_j))^2} = \sqrt{\sum_{k=1}^C\left(c^1_k (s_i - s_{i+1})\right)^2} \\
&=  |s_i-s_{i+1}|\cdot \|c^1\|_2 = |s_i-s_{i+1}|.
\end{aligned}
\end{equation*}
On the other hand, it follows that
\begin{equation*}
\begin{aligned}
\sum_{k=1}^C z^1_k(x_i,y_j) D_x u_k(x_i,y_j) &= \sum_{k=1}^C (c_k^1)^2 l^1(x_i) D_xD_x^Tl^1(x_i)\\
&= |s_i - s_{i+1}| \cdot\|c^1\|_2^2 = |s_i - s_{i+1}|.
\end{aligned}
\end{equation*}
Therefore, we have obtained $\sqrt{\sum_{k}(D_x u_k(x_i,y_j))^2} = \sum_{k} z^1_k(x_i,y_j)  D_x u_k(x_i,y_j)$ and, similarly, $\sqrt{\sum_{k}(D_y u_k(x_i,y_j))^2} = \sum_{k} z^2_k(x_i,y_j)  D_y u_k(x_i,y_j)$. These equalities prove that $\langle \mathbf{z}, D\mathbf{u}\rangle = \|D\mathbf{u}\|_{2,1,1}$, which yields the result.
\end{proof}

\begin{theorem}\label{thm:singularInfty11}
Let $l^1$ and $l^2$ be discretizations of arbitrary piecewise linear functions in $[-1,1]$, with the properties described previously. For each $k\in\{1,\ldots, C\}$, consider $c^1_k, c^2_k \in \{0, \pm 1\}$ and define
$$
z_k^1(x_i,y_j) = \left\lbrace \begin{array}{ll} \dfrac{c^1_k}{\|c^1\|_0} l^1(x_i) & \text{if } \|c^1\|_0\neq 0, \\ 0 & \text{otherwise,} \end{array}\right.
$$
and
$$
z_k^2(x_i,y_j) = \left\lbrace \begin{array}{ll}\dfrac{c^2_k}{\|c^2\|_0}l^2(y_j) & \text{if } \|c^2\|_0\neq 0, \\ 0 & \text{otherwise.} \end{array}\right.
$$
Let us also define $u_k(x_i,y_j) =(D_x^T z^1_k)(x_i,y_j) + (D_y^T z^2_k)(x_i,y_j)$. Then, $\mathbf{u}\in\partial J(\mathbf{u})$ for $J(\mathbf{u}) = \|D\mathbf{u}\|_{\infty,1,1}$.
\end{theorem}

\begin{proof}
Once more, we need to show that $\|\mathbf{z}\|_{1,\infty,\infty} \leq 1$, which follows from $|l^r(x)|\leq 1$ and $\sum_k |c_k^r| = \|c^r\|_0$ due to $c_k^r\in\{0,\pm 1\}$, and $\langle \mathbf{z}, D\mathbf{u}\rangle = \|D\mathbf{u}\|_{\infty,1,1}$. The latter is achieved if $\max_k |D_x u_k(x_i,y_j)| = \sum_{k} z^1_k(x,y)  D_x u_k(x_i,y_j)$. In the nontrivial case, $\|c^1\|_0\neq 0$, we obtain
\begin{equation*}
\begin{aligned}
\max_{1\leq k\leq C} |D_x u_k(x_i,y_j)| &= \max_{1\leq k \leq C} \left| D_xD_x^Tz_k^1(x_i,y_j)\right| = \max_{1\leq k\leq C} \left|\dfrac{c^1_k}{\|c^1\|_0} D_xD_x^Tl^1(x_j)\right|\\
&= \max_{1\leq k\leq C} \left|\dfrac{c^1_k}{\|c^1\|_0} (s_i - s_{i+1})\right| = \dfrac{|s_i - s_{i+1}|}{\|c^1\|_0},
\end{aligned}
\end{equation*}
as well as 
\begin{equation*}
\begin{aligned}
 \sum_{k=1}^C z^1_k(x,y)  D_x u_k(x_i,y_j) &=  \sum_{k=1}^C \frac{\left(c^1_k\right)^2}{\|c^1\|_0^2} l^1(x_i)D_xD_x^T l^1(x_i) = \dfrac{|s_i-s_{i+1}|}{\|c^1\|_0^2} \sum_{k=1}^C \left(c_k^1\right)^2\\
&=  \dfrac{|s_i-s_{i+1}|}{\|c^1\|_0^2} \|c^1\|_0 =  \dfrac{|s_i - s_{i+1}|}{\|c^1\|_0},
\end{aligned}
 \end{equation*}
where we used $\sum_k (c_k^1)^2=\|c^1\|_0$ because of $c_k^r\in\{0,\pm 1\}$. Similarly, one shows that $\max_k |D_y u_k(x_i,y_j)| = \sum_{k} z^2_k(x,y) D_y u_k(x_i,y_j)$, which ends the proof.
\end{proof}



\section{Proof of Theorem \ref{th:chainRule}}\label{app:ThmSubdif}

We first prove \eqref{eq:chainRule1}. Let $\xi = \sum_{j} q_j v_j$ be, with $q=(q_1,\ldots q_m) \in \partial g (f(x_0))$ and $v_j \in \partial f_j(x_0)$. From $q\in \partial g(f(x_0))$ one has that
$$
g(f(x)) \geq g(f(x_0)) + \langle q, f(x)-f(x_0)\rangle,  \quad \forall  x\in\R^n,
$$
and each condition $v_j\in \partial f_j(x_0)$ yields
$$
f_j(x) \geq f_j(x_0) + \langle v_j, x-x_0\rangle, \quad \forall x\in\R^n.
$$
By using the above inequalities, we finally obtain
$$
(g\circ f)(x) \geq g(f(x_0)) + \sum_{j=1}^m q_j \langle v_j, x -x_0\rangle = (g\circ f)(x_0) + \langle \xi, x-x_0\rangle, \quad \forall x \in\R^n.
$$

Assume now that $x_0\in\inter\dom (g\circ f)$ and $f_j$ is locally l.s.c for all $j\in\{1,\ldots, m\}$. Let $\partial$, $\widehat{\partial}$ and $\partial^{\infty}$ denote the regular, general and horizon subdifferentials \cite{Rockafellar1998}, respectively. For $q\in\R^m$, we introduce the notation $(qf)(x):=\sum_{j} q_j f_j(x)$. Since $f$ and $g$ are proper and convex functions, \cite[Proposition 8.12]{Rockafellar1998} implies $\widehat{\partial}(qf)(x_0)=\partial(qf)(x_0)$, $\widehat{\partial}g(f(x_0)) = \partial g(f(x_0))$, and $ \widehat{\partial} (g\circ f)(x_0)=\partial (g\circ f)(x_0)$. Furthermore, the properties of $f_j$ allow applying the same proposition to see
$$
\partial^{\infty} f_j(x) = \left\lbrace \xi\in\R^n : \langle \xi, y-x\rangle \leq 0, \: \forall y\in\dom f_j\right\rbrace =\{ 0\}, \quad\forall x\in\R^n,
$$
which is equivalent to $f$ being strictly continuous by \cite[Theorem 9.13]{Rockafellar1998}. Note that  \cite[Proposition 8.12]{Rockafellar1998} also yields
$$
\partial^{\infty} g(f(x_0)) = \left\lbrace \xi\in\R^m : \langle \xi, y-f(x_0)\rangle \leq 0, \: \forall y\in\dom g\right\rbrace,
$$ 
from where one deduces $\partial^{\infty}g(f(x_0))=\{ 0\}$ due to $x_0\in\inter\dom (g\circ f)$. Putting it all together, \cite[Theorem 10.49]{Rockafellar1998} applies and so
\begin{equation*}
\begin{aligned}
\widehat{\partial}\left(g\circ f\right)(x_0) &\subset \bigcup \left\{ \widehat{\partial}(qf)(x_0) \: : \: q\in\widehat{\partial} g\left(f(x_0)\right)\right\},\\
\partial^{\infty}\left(g\circ f\right)(x_0) &\subset \bigcup \left\{ \widehat{\partial}(qf)(x_0) \: : \: q\in\partial^{\infty} g\left(f(x_0)\right)\right\}.
\end{aligned}
\end{equation*}
Finally, the previous two inclusions along with the equivalence of regular and general subdifferentials lead to the equality in \eqref{eq:chainRule1}.

\end{document}